\documentclass[10pt,twocolumn,letterpaper]{article}

\usepackage[pagenumbers]{cvpr} 

\definecolor{cvprblue}{rgb}{0.21,0.49,0.74}
\usepackage[pagebackref,breaklinks,colorlinks,allcolors=cvprblue]{hyperref}

\usepackage{graphicx}
\usepackage{amsmath, amsthm}
\usepackage{amssymb}
\usepackage{booktabs}
\usepackage{enumitem}
\usepackage{mathtools}

\usepackage{multirow}
\usepackage{xcolor}
\usepackage{colortbl}
\usepackage{makecell}
\usepackage{soul}
\usepackage{pifont}

\newcounter{todocounter}

\usepackage[linesnumbered,ruled]{algorithm2e}

\DeclareMathOperator{\sgn}{sgn}

\newtheorem{lemma}{Lemma}
\newtheorem{definition}{Definition}
\newtheorem{assumption}{Assumption}

\newtheorem{remark}{Remark}

\newcommand{\eqautoref}[1]{\hyperref[#1]{Equation (\ref{#1})}}
\newcommand{\ineqautoref}[1]{\hyperref[#1]{Inequality (\ref{#1})}}

\newcommand{\algautoref}[1]{\hyperref[#1]{Algorithm~\ref{#1}}}
\newcommand{\apxautoref}[1]{\hyperref[#1]{Appendix Section~\ref{#1}}}

\def\paperID{11921} 
\def\confName{CVPR}
\def\confYear{2025}

\title{Detecting Backdoor Attacks in Federated Learning via Direction Alignment Inspection}

\author{Jiahao Xu \hspace{3em} Zikai Zhang \hspace{3em} Rui Hu\\
University of Nevada, Reno \\
\texttt{\{jiahaox,zikaiz,ruihu\}@unr.edu}
}

\begin{document}
\maketitle
\begin{abstract}
The distributed nature of training makes Federated Learning (FL) vulnerable to backdoor attacks, where malicious model updates aim to compromise the global model's performance on specific tasks. Existing defense methods show limited efficacy as they overlook the inconsistency between benign and malicious model updates regarding both general and fine-grained directions. To fill this gap, we introduce AlignIns, a novel defense method designed to safeguard FL systems against backdoor attacks. AlignIns looks into the direction of each model update through a direction alignment inspection process. Specifically, it examines the alignment of model updates with the overall update direction and analyzes the distribution of the signs of their significant parameters, comparing them with the principle sign across all model updates. Model updates that exhibit an unusual degree of alignment are considered malicious and thus be filtered out. We provide the theoretical analysis of the robustness of AlignIns and its propagation error in FL. Our empirical results on both independent and identically distributed (IID) and non-IID datasets demonstrate that AlignIns achieves higher robustness compared to the state-of-the-art defense methods. The code is available at \url{https://github.com/JiiahaoXU/AlignIns}.
\end{abstract}    

\section{Introduction}
Unlike traditional centralized training methods, which require gathering and processing all data at a central location such as a server, Federated Learning (FL)~\cite{FL_OG_shake}, as a decentralized training paradigm, allows a global model to learn from data distributed across various local clients, thereby achieving the goal of privacy-preserving. During training, the server distributes the global model to local clients, and each client trains the received global model using its local dataset, and then submits its local model update to the server for global model refinement. FL has been applied in various fields, including healthcare\cite{healthcare}, finance~\cite{fin1}, and remote sensing~\cite{remotesense}, where local data privacy is essential.

Although promising, the distributed nature of FL systems makes them vulnerable to a range of advanced poisoning attacks~\cite{bulyan, liu2022evil, tian2022comprehensive}. This vulnerability primarily stems from the server's lack of close monitoring of the local data and the training algorithm on clients. Consequently, this drawback allows attackers to compromise the data of local clients or interfere with the training algorithm, enabling them to inject malicious local model updates that distort the performance of the global model. For example, backdoor attacks~\cite{scaling, dba, badnet, Neurotoxin, PGD} have gained significant attention due to their stealthiness and practical effectiveness. In detail, backdoor attacks in FL seek to preserve the performance of the global model on clean inputs (\ie, the main task), while inducing the global model to make incorrect predictions on inputs that contain a certain predefined feature (\ie, the backdoor task). As backdoor attacks maintain the main task and the backdoor task simultaneously, the malicious local model updates are statistically similar to benign ones~\cite{PGD, flame} (\textit{poison-coupling effect}~\cite{lockdown}), making anomaly detection more challenging on the server side.

Existing defense methods (\ie, aggregation rule) usually aim to identify malicious model updates and filter them out to achieve better robustness by magnitude-based metrics extracted from local model updates (\eg, Manhattan distance~\cite{mm, mesas} and Euclidean distance~\cite{krum, bulyan}). However, magnitude-based metrics are ineffective in distinguishing stealthy backdoor attacks where benign and malicious model updates are usually similar in magnitude. Additionally, when the global model tends to converge, the magnitude of each model update becomes very small, making effective malicious manipulation on magnitude negligible. To this end, some works employ Cosine similarity to check the pair-wise directional information of model updates~\cite{Foolsgold, deepsight, flame}. However, pair-wise Cosine similarity between two model updates only captures their general directional similarity and overlooks fine-grained information (\eg, signs of parameters), resulting in limited robustness. In addition, in FL settings with non-IID data, the pair-wise Cosine similarity of model updates can be easily perturbed by the naturally diverse benign model updates. \textit{Furthermore, there is a deficiency in theoretical analysis within the literature concerning the effects of data heterogeneity on defense methods deployed by the server in FL.}

In this work, we propose a novel defense method designed to defend against backdoor attacks in FL, named \textbf{AlignIns} (Direction \underline{\textbf{Align}}ment \underline{\textbf{Ins}}pection), which examines local model updates for directional alignment at different granularity levels to identify malicious updates. Specifically, after receiving all model updates from clients, AlignIns evaluates each update by (1) inspecting temporal directional alignment with the global model of the latest round with \textit{Cosine similarity} and (2) assessing more fine-grained sign alignment with the principal sign across all updates with a novel metric \textit{sign alignment ratio}. Particularly, when calculating the sign alignment ratio, AlignIns focuses on the signs of important parameters in each update to accurately capture alignment information. Using these two directional metrics, AlignIns performs anomaly detection with the robust $\mathrm{MZ\_score}$ which requires minimal hyperparameters to filter updates with unusual directional patterns out. Finally, AlignIns clips the remaining updates to mitigate the impact of updates with abnormally large magnitudes. We also provide a theoretical analysis of AlignIns' robustness and its propagation error in FL. The main contributions of this work are three folds:
\begin{itemize}
    \item We present a novel defense method, AlignIns, to defend against backdoor attacks in FL. \textbf{\textit{To the best of our knowledge, AlignIns is the first defense method in FL that analyzes the directional patterns of local model updates at different levels of granularity.}} AlignIns is fully compatible with existing FL frameworks.
    \item \textbf{\textit{To the best of our knowledge, we provide the first theoretical robustness analysis for a filtering-based defense method against backdoor attacks under non-IID data in FL.}} 
    Moreover, we prove that the propagation error of AlignIns is bounded during the training of FL.
    \item We empirically evaluate the effectiveness of AlignIns through extensive experiments on both IID and non-IID datasets against various state-of-the-art (SOTA) backdoor attacks. Compared to existing SOTA defense methods, AlignIns exhibits superior robustness.
\end{itemize}

\section{Background and Related Works}

\textbf{Federated Learning.} In a typical FL system, a central server controls a set of $n$ clients to train a global model $\theta \in \mathbb{R}^d$ collaboratively. The objective of FL is to solve the following optimization problem:
$
    \min_{\theta} (1/n)\sum_{i=1}^n \mathcal{L}_i(\theta; \mathcal{D}_i),
$
where $\mathcal{L}_i(\cdot)$ denotes the learning objective specific to client $i$ and $\mathcal{D}_i$ denotes the local dataset for client $i$. The commonly used method to solve this problem iteratively is FedAvg~\cite{fedavg}. In detail, 
at round $t$ of FedAvg, each client $i\in[n]$ downloads the current global model $\theta^{t}$, updates it by optimizing its local objective, resulting in $\theta_i^{t}$, and transmits its model update $\Delta_i^t = \theta_i^{t} - \theta^{t}$ to the server. The server then refines the global model by averaging these updates as follows: $\theta^{t+1} = \theta^{t} + (1/n)\sum_{i=1}^n \Delta_i^t$. This process continues until the global model reaches convergence.

\textbf{Backdoor attacks in FL.} Empirical evidence has shown that FL is vulnerable to backdoor attacks~\cite{badnet, scaling, PGD, dba, Neurotoxin, f3ba, a3fl, iba, krauss2024automatic} due to its lack of access to local training data~\cite{scaling}. For instance,
\textit{Projected Gradient Descent} (PGD) attack~\cite{PGD} periodically projects the local model onto a small sphere centered around the global model from the previous training round, with a predefined radius. 
\textit{Distributed Backdoor Attack} (DBA)~\cite{dba} decomposes the centralized trigger into several smaller, distributed local triggers. Each poisoned client uses one of these local triggers, but during testing, the adversary injects the full trigger into the test samples. Recently, research has focused on trigger-optimization backdoor attacks~\cite{a3fl, f3ba, iba, lyu2023poisoning, alam2022perdoor}, which aim to search optimized triggers to enhance the effectiveness and stealthiness.

\textbf{Defending against backdoor attacks in FL.} \label{backdoor_defense} Generally, based on how defense methods mitigate the impact of malicious updates, existing defense methods can be categorized into \textit{filtering-based methods}~\cite{krum, deepsight, fltrust, mesas, mm, flame, lasa, masa} and \textit{influence-reduction methods}~\cite{geomed, rieger2022crowdguard, Foolsgold, rlr, lockdown}.

\textit{1) influence-reduction methods} aim to integrate all model updates but employ strategies to reduce the impact of malicious updates. For instance,
\textit{RFA}~\cite{geomed} is proposed to use the geometric median of local models as the aggregation result, under the assumption that malicious models significantly deviate from benign models. 
\textit{Foolsgold}~\cite{Foolsgold} assumes that the malicious updates are consistent with each other. It assigns aggregation weights to model updates based on the maximum Cosine similarity between the last layers of pairwise model updates. A higher Cosine similarity value indicates a higher probability that the updates are malicious, leading to smaller aggregation weights being assigned. 
The effectiveness of influence-reduction methods is inherently limited because they cannot eliminate the impact of malicious activity, leading to a significant risk of compromise.

\textit{2) Filtering-based methods} aim to detect and remove malicious local model updates before aggregation thus attempting to achieve the highest robustness. For example,
\textit{Multi-Krum}~\cite{krum} selects the multiply reliable local model updates for aggregation by identifying the one with the smallest sum of squared Euclidean distances to all other updates. 
\textit{Multi-Metrics}~\cite{mm} explores the combination of Manhattan distance, Euclidean distance, and Cosine similarity for each update to collaboratively filter out outliers. 
However, due to the dual objectives of backdoor attacks—that is, maintaining accuracy on the main task while maximizing accuracy on the backdoor task—malicious updates must mimic benign model updates, important weights for the main task typically have large values and can dominate the magnitude of a model update. As a result, magnitude-based detection methods become ineffective against backdoor attacks. Additionally, methods that rely solely on Cosine similarity also show limited effectiveness since they capture general directional alignment and overlook finer-grained information.

\section{Our Solution: AlignIns}

Our method, AlignIns, detailed in \algautoref{alg:main}, mitigates the impact of malicious updates through a two-step process. First, \textit{direction alignment inspection} is applied to examine each local model update comprehensively in terms of direction. Second, \textit{post-filtering model clipping} is used to further enhance the robustness of AlignIns on defending potential magnitude-based attack methods before final aggregation.

\setlength{\textfloatsep}{10pt}
\begin{algorithm}[t]
    \caption{AlignIns}
    \label{alg:main}
    \KwIn{Set of $n$ local model updates $\{\Delta_i^{t}\}_{i=1}^{n}$ where $m$ of them are malicious, global model $\theta^t$, $\mathrm{TDA}$ radius $\lambda_{c}$, $\mathrm{MPSA}$ radius $\lambda_{s}$, extraction parameter $k$}
    \KwOut{Aggregated model update $\widetilde{\Delta}$}
    
    Initialize benign set $\mathcal{S} \leftarrow \emptyset$\
    
     $ \omega \leftarrow \{\mathrm{TDA}(\Delta_i^{t}, \theta)\}_{i=1}^n \hfill\lhd$  by \eqautoref{tda} \label{alg:tda}
    
     $p  \leftarrow \sgn(\sum_{i=1}^n \sgn(\Delta_i^{t}))$ \label{alg_MPSA_1}
    
     $ \rho \leftarrow \{\mathrm{MPSA}(\Delta_i^{t}, p, k)\}_{i=1}^n \hfill\lhd$ by \eqautoref{local_sar} \label{alg_MPSA_2}
    
    \For{$i \in [n]$}{
         $\lambda_{i,c} \leftarrow \textit{$\mathrm{MZ\_score}$}(\omega_i, \omega) \hfill\lhd$ by \eqautoref{eq_mzscore} \label{alg_mzscore_1}
         
         $\lambda_{i,s} \leftarrow \textit{$\mathrm{MZ\_score}$}(\rho_i,\rho) \hfill\lhd$ by \eqautoref{eq_mzscore} \label{alg_mzscore_2}
         
        \If{$|\lambda_{i,c}| \leq \lambda_{c}$ and $|\lambda_{i,s}| \leq \lambda_s$}{\label{alg_mzscore_3}
            $\mathcal{S} \leftarrow \mathcal{S} \cup \{i\}$\ \label{alg_mzscore_4}
        }
    }
    
     $c \leftarrow \mathrm{med} (\{\| \Delta_i^{t} \|\}_{i \in \mathcal{S}})$  \label{alg_calculate_median}
     
     $ \widetilde{\Delta} \leftarrow (1 / |\mathcal{S}|)\sum_{i\in \mathcal{S}}\left({\Delta}_i^{t} \cdot \min\{1, c/\| \Delta_i^{t} \| \} \right)$ \label{alg_aggregate_with_rescaling}
    
    \Return $\widetilde{\Delta}$
\end{algorithm}

\textbf{Direction alignment inspection.} Existing defense methods against backdoor attacks in FL primarily focus on examining the magnitude (\eg, Manhattan distance and Euclidean distance) and the overall direction (\eg, Cosine similarity) of model updates. However, backdoor attacks are designed to maintain the main task accuracy, making the magnitude difference between malicious and benign updates nearly indistinguishable. Additionally, advanced attacks such as PGD~\cite{PGD} and Lie~\cite{lie} attacks are specifically crafted to bypass magnitude-based defenses. Therefore, AlignIns focuses on direction-based analysis to identify suspect updates, using two processes described below.

\textit{1) Temporal direction alignment checking:} Since malicious clients need to maintain both the main task and the backdoor task, the optimization direction of a malicious local model tends to deviate from that of benign models. AlignIns leverages this deviation and performs a Temporal Direction Alignment ($\mathrm{TDA}$) checking, which calculates the Cosine similarity between a local update and the latest global model (\autoref{alg:tda} in \algautoref{alg:main}) to assess the general alignment level of each local update. Formally, the $\mathrm{TDA}$ value $\omega_i$ of a local model update $\Delta_i^{t}$ is calculated as 
\begin{align}
    \omega_i \coloneqq \langle \Delta^{t}_i,\theta^t \rangle / (\|\Delta^{t}_i\| \|\theta^t\|). \label{tda}
\end{align}
We use local model updates rather than local models because our goal is to measure how closely each client's updates align with the direction of the global model. Local model updates specifically capture these incremental adjustments. 
Notably, malicious clients tend to exhibit similar $\mathrm{TDA}$ values, which differ from those of benign clients, creating an opportunity for detection. It is important to note that while the magnitude of model updates typically decreases as the global model converges, the $\mathrm{TDA}$ value does not follow the same trend. Consequently, magnitude-based anomaly detection becomes progressively less effective throughout training due to the decreasing magnitude. In contrast, the variability in $\mathrm{TDA}$ values continues to be useful for identifying malicious behavior.

\textit{2) Masked principal sign alignment checking:}
In backdoor attacks where the manipulations are stealthy, subtle malicious directional information can easily blend into the parameters of models with large magnitude, especially for models with large dimensions, which makes the $\mathrm{TDA}$ less useful under strong backdoor attacks since the $\mathrm{TDA}$ captures the overall directional information. Therefore, in addition to the $\mathrm{TDA}$, we look into the signs of parameters to provide a finer-grained directional assessment of local model updates. The signs of a vector represent its coordinate-wise direction. In the context of backdoor attacks, the distributions of the signs of malicious model updates differ from those of benign updates. This is particularly significant when the model is close to convergence, at which point the magnitude of model updates becomes very small, making large manipulations on magnitudes impractical. Therefore, manipulation of the direction, or the signs of parameters, can emerge as a more significant and effective strategy. Several works also utilize the signs of models for enhancing backdoor robustness. For example, RLR~\cite{rlr} assigns an opposite global learning rate to a coordinate of the averaged model update if the signs on this coordinate do not consistently align with the majority across all updates. SignGuard~\cite{signguard} calculates the proportions of positive, zero, and negative signs for each model update as the input of a clustering algorithm to identify malicious model updates. 
However, these methods utilize the signs of all parameters in the model update, regardless of their significance. Consequently, the performance of sign-based metrics can be significantly impacted by those many unimportant parameters, especially for large DNN models, leading to an inaccurate representation of the model update’s direction.

To this end, AlignIns utilizes a Masked Principle Sign Alignment ($\mathrm{MPSA}$) checking to inspect the sign alignment degree between the important parameters of each local update and a well-designed principle sign of all local updates. Specifically, to construct the principle sign over local updates, for each coordinate of local updates, we take the majority of the signs across all model updates as the principal sign of this coordinate, which can be mathematically formulated as 
$
    p := \sgn\left(\sum_{i=1}^n \sgn(\Delta_i^{t})\right),\label{psa}
$
where $p \in \mathbb{R}^d$ represents the vector of principal signs and $\sgn(\cdot)$ is the function to take the signs of a vector. Note that the principal sign represents sign-voting results for each coordinate, making it stand for the major direction/dynamic for each coordinate. With this principle sign over local updates, we inspect the alignment of the signs of important parameters of each model update with it. More specifically, we use a Top-$k$ indicator defined as follows to identify the $k$ most important parameters that have the largest absolute values in a vector.  
\begin{definition}[Top-$k$ Indicator $\mathrm{Top_k} ( \cdot )$]\label{def:spar}
    For a vector $x \in \mathbb{R}^d$ and a masking parameter $k$, where $1\leq k \leq d$ , the Top-$k$ indicator $\mathrm{Top_k} ( \cdot )$: $\mathbb{R}^d \rightarrow \mathbb{R}^d$ is defined as $[\mathrm{Top_k} ( x)]_j = 1$ if $[x]_j \in \xi$ and $[\mathrm{Top_k} ( x )]_j = 0$ otherwise,
    where $\xi = \{|x_{\pi(1)}|, |x_{\pi(2)}|, \dots, |x_{\pi(k)}|\}$, here $\pi$ is a permutation of $[d]$ such that $|x_{\pi(i)}| \geq |x_{\pi(i+1)}|$ for all $1 \leq i <d$.
\end{definition}
The Top-$k$ indicator $\mathrm{Top_k} ( \cdot )$ takes each local model update as input and outputs a mask vector in which each element is either 1 or 0 with the same size as the input. To quantify the alignment in sign distributions of each local model update and the principle sign, we define the Sign Alignment Ratio ($\mathrm{SAR}$) as follows.
\begin{definition}[Sign Alignment Ratio]\label{def_sar}
    For vectors $x \in \mathbb{R}^d$ and $y \in \mathbb{R}^d$, the sign alignment ratio $\rho$ of $x$ to $y$ is defined as
    $\rho := 1 - \|\sgn(x) - \sgn(y) \|_0/d$
where $\| \cdot \|_0$ is $L_0$-norm.
\end{definition}
Here, $ \rho\in[0,1]$ and a larger $\rho$ indicate a higher degree of alignment between the signs of $x$ and $y$. Combining $\mathrm{Top_k} ( \cdot )$ and $\mathrm{SAR}$, we have the $\mathrm{MPSA}$ value $\rho_i$ for local update $\Delta^{t}_i$ formulated as follows:
\begin{equation}
    \rho_i := 1 - \left\| \left(\sgn(\Delta_i^{t})  - p\right) \odot  \mathrm{Top}_k (\Delta_i^{t})\right\|_0 / k,\label{local_sar} 
\end{equation}
where $\odot$ is the Hadamard product, $\sgn(\Delta_i^{t}) - p$ computes a sign difference vector, capturing the difference between the sign of $\Delta_i^{t}$ and the principal reference sign $p$. Since $\mathrm{MPSA}$ checking focuses on the important parameters,  this difference vector is element-wise multiplied with the Top-$k$ mask derived from $\Delta_i^{t}$, effectively setting unimportant coordinates to zero. The $L_0$-norm is then applied to count the not-aligned elements and with the masking parameter $k$ to ultimately determine the $\mathrm{SAR}$. $\mathrm{MPSA}$ checking effectively reveals malicious local updates by combining both magnitude and directional information from model updates, allowing for clear differentiation between malicious and benign updates. AlignIns calculates the $\mathrm{MPSA}$ value for each update with the principal sign iteratively (\autoref{alg_MPSA_1}--\ref{alg_MPSA_2}) and forward them to the following anomaly detection process.

\textit{3) Efficient anomaly detection with $\mathrm{MZ\_score}$:} W apply robust filtering to remove updates with abnormal $\mathrm{TDA}$ and $\mathrm{MPSA}$ values. Specifically, 
we use the robust standardization metric named the \textit{Median-based Z-score} ($\mathrm{MZ\_score}$)~\cite{lasa, masa}, detailed in \autoref{def:mz}, which is a variant of the traditional \textit{Z-score} standardization metric.

\begin{definition}[$\mathrm{MZ\_score}$]\label{def:mz}
    For a set of values $X \coloneqq \{x_1, \dots, x_n \}$ with median $\mathrm{med}(X)$ and standard deviation $\sigma$, the $\mathrm{MZ\_score}$ $\lambda_i$ of any $x_i \in X$ is defined as 
    \begin{equation}
    \lambda_i := (x_i - \mathrm{med}(X)) / \sigma. 
    \label{eq_mzscore}
\end{equation}
\end{definition}
$\mathrm{MZ\_score}$ calculates the number of standard deviations an element is from the median, which may be either positive or negative. In AlignIns, the $\mathrm{MZ\_score}$s for $\mathrm{TDA}$ and $\mathrm{MPSA}$ values are computed for each local update (\autoref{alg_mzscore_1}--\ref{alg_mzscore_2}). Those with high absolute $\mathrm{MZ\_score}$s (\ie, outliers) are excluded using two predetermined filtering radii: $\lambda_c$ for $\mathrm{TDA}$ and $\lambda_s$ for $\mathrm{MPSA}$ (\autoref{alg_mzscore_3}--\ref{alg_mzscore_4}). The use of the $\mathrm{MZ\_score}$ allows for the adaptation to the varying range of $\mathrm{TDA}$ and $\mathrm{MPSA}$ values during training, requiring only minimal hyper-parameters. Additionally, by configuring $\lambda_c$ and $\lambda_s$, we can manage the trade-off between the robustness and main task accuracy of AlignIns. For example, when robustness is the primary concern in the FL, choosing small $\lambda_c$ and $\lambda_s$ values is essential to attain the highest robustness. 

\textbf{Post-filtering model clipping.}
After filtering, the remaining clients, considered benign, are included in the set $\mathcal{S}$ (\autoref{alg_mzscore_4}) and contribute to the model averaging process. 
However, since our filtering primarily focuses on the direction of model updates (although $\mathrm{MPSA}$ does consider magnitude when using the Top-$k$ indicator), there is a risk that it might overlook updates with large magnitudes, such as those updates generated by Scaling attack~\cite{scaling}. To this end, AlignIns re-scales model updates in $\mathcal{S}$ by using the median of the $L_2$-norms of these updates as a clipping threshold and aggregates the clipped model updates as the global model update $\widetilde{\Delta}$ (\autoref{alg_calculate_median}--\ref{alg_aggregate_with_rescaling}). It is worth noting that performing clipping before filtering does not affect the filtering results. However, clipping after filtering enhances robustness, as the clipping threshold is more likely determined by benign updates. We discuss the computational cost of AlignIns and compare it with other baselines in \apxautoref{apdx: computational_cost}.

\section{Robustness and Propagation Error Analysis}

In this section, we conduct a theoretical analysis of the robustness of AlignIns, as well as its propagation error in FL. Before presenting our theoretical results, we make the following assumptions. Note that \autoref{ass1}--\ref{ass2} are commonly used in the theoretical analysis of distributed learning systems~\cite{hu2023federated, signguard, nesterov2018lectures}. \autoref{ass3} states a standard measure of inter-client heterogeneity in FL~\cite{kappa, karimireddy2021byzantine, el2021collaborative}. This heterogeneity complicates the problem of FL with backdoor adversaries, as it may cause the server to confuse malicious updates with flawed model updates from benign clients holding outlier data points~\cite{kappa}. 

\begin{assumption}[$\mu$-smoothness~\cite{nesterov2018lectures}] \label{ass1}
    Each local objective function $\mathcal{L}_{i}$ for benign client $i\in\mathcal{B}$ is $\mu$-Lipschitz smooth with $\mu >0$, \ie, for any $x,y \in \mathbb{R}^d$,
    $
        \left \| \nabla \mathcal{L}_{i}(x) - \nabla \mathcal{L}_{ i}(y) \right\| \leq \mu \left \| x-y \right\|, \forall i\in\mathcal{B},
    $
    which further gives:
    $
        \mathcal{L}_{i}(x) - \mathcal{L}_{ i}(y) \leq \nabla \mathcal{L}_{ i}(x)^T(y-x) + (\mu/2)\left \| x-y \right \|^2, \forall i\in\mathcal{B}.
    $
\end{assumption}

\begin{assumption}[Unbiased gradient and bounded variance]\label{ass2}
    The stochastic gradient at each benign client is an unbiased estimator of the local gradient, \ie, $\mathbb{E}[g_i(x)] = \nabla \mathcal{L}_i(x)$ and has bounded variance, \ie, for any $ x\in\mathbb{R}^d$,
$
    \mathbb{E} \left\|g_i(x) - \nabla \mathcal{L}_{i}(x)) \right\|^2 \leq \nu_i^2, \forall i\in\mathcal{B},
    $
where the expectation is over the local mini-batches. We also denote
$\bar{\nu} \coloneqq \left(1/ |\mathcal{B}|\right)\sum_{i\in\mathcal{B}}\nu^2_i$ for convenience.
\end{assumption}

\begin{assumption}[Bounded heterogeneity]\label{ass3} There exist a real value $\bar{\zeta}$ such that for any $x \in \mathbb{R}^d$,
$
    (1/|\mathcal{B}|)\sum_{i\in\mathcal{B}}\left \| \nabla \mathcal{L}_{i}(x) - \nabla \mathcal{L}_{\mathcal{B}}(x) \right \|^2 \leq \bar{\zeta},
$
where the $\nabla \mathcal{L}_\mathcal{B}(x) \coloneqq \left(1/ |\mathcal{B}|\right)\sum_{i\in\mathcal{B}}\mathcal{L}_{i}(x)$.
\end{assumption}

Note that these assumptions apply to benign clients only since malicious clients do not need to follow the prescribed local training protocol of FL. 

\subsection{Robustness Analysis of AlignIns}
To theoretically evaluate the efficacy of a filtering-based defense method like AlignIns, we introduce the concept of $\kappa$-robust filtering~\cite{lasa} as defined in \autoref{f_kappa_def}. Note that \autoref{f_kappa_def} is similar to \textit{$(f, \kappa)$-robustness} defined in~\cite{kappa, allouah2023robust}, \textit{$(\delta_{\max}, c)$-ARAgg} defined in~\cite{karimireddy2021byzantine, gorbunov2022variance, malinovsky2023byzantine}, and \textit{$(f, \lambda)$-resilient averaging} defined in~\cite{farhadkhani2022byzantine}. Our robustness definition adopts a constant upper bound and focuses on quantifying the distance between the output of a filtering-based defense method and the average of all benign updates, which represents the optimal output of such a rule.

\begin{definition}[$\kappa$-robust filtering~\cite{lasa}] \label{f_kappa_def} 
A filtering-based aggregation rule $F\colon \mathbb{R}^{d\times n} \rightarrow \mathbb{R}^d$ is called $\kappa$-robust if for any vectors $ \{ x_1, \dots, x_n \} \in \mathbb{R}^d$ and a benign set $\mathcal{B} \subseteq [n]$ of size $n - m$, the output $\hat{x} \coloneqq F(x_1, \dots, x_n)$ satisfies
$
    \left \| \hat{x} - \bar{x}_\mathcal{B} \right \|^2 \leq \kappa,
$
where $\bar{x}_\mathcal{B} \coloneqq (1/|\mathcal{B}|) \sum_{i\in \mathcal{B}} x_i$, and $\kappa \geq 0$ refers to the \textit{robustness coefficient} of $F$.
\begin{remark}
The $\kappa$-robust filtering guarantees that the error of a filtering-based aggregation rule in estimating the average of the benign inputs is upper-bounded by a constant $\kappa$. This measure provides a quantitative way to assess the robustness of the filtering-based aggregation rule. A smaller $\kappa$ indicates a smaller discrepancy between the empirical output and the optimal output of $F$. If $F$ identifies and removes all malicious inputs and keeps all benign inputs, we have $\kappa = 0$, achieving the highest level of robustness.
\end{remark}
\end{definition}

Based on \autoref{f_kappa_def}, we prove that the proposed AlignIns, when applied to $n$ input models, of which $m$ are malicious, satisfies $\kappa$-robust filtering with $\kappa = O(1 + m / (n-2m))$, as stated in \autoref{lemma:AlignInskappa}. 

\begin{lemma}[$\kappa$-robustness of AlignIns] \label{lemma:AlignInskappa}
Under \autoref{ass2}--\ref{ass3}, assume $n> 1$, $ 0 \leq m < n/(3+\epsilon)$ with a positive constant $\epsilon$, AlignIns satisfies $\kappa$-robust filtering with 
\begin{align*}
    \kappa & = \left( 1 + m / \left( n-2m \right) \right) \left( \left(2/\epsilon + 1\right) \left(2\bar{\nu} + \bar{\zeta}\right) + 8 c^2 \right) \\
     &= O\left( 1 + m / \left( n-2m \right)  \right), \label{ourkappa_order}
\end{align*}
if the local learning rate satisfies $\eta\leq 1/2\tau$ and there exist two sufficiently large filtering radii such that $|\mathcal{S}| \geq n-2m$. Here, $\bar{\nu}$ and $\bar{\zeta}$ represent the gradient variance and local divergence, respectively; $c$ is the clipping threshold.
    \begin{proof}
    The proof is given in \apxautoref{apdx: proof_lemma_1}.
\end{proof}
\begin{remark}
The condition on $\mathcal{S}$ highlights the importance of selecting appropriate filtering radii. These radii cannot be zero or too small; otherwise, only the median or a few model updates will be averaged to update the global model. This can lead to a performance drop due to the lack of model updates. Moreover, the model clipping threshold $c$ can effectively control the magnitude of potential malicious updates in the selection set, thus preventing $\kappa$ from exploding due to updates with large magnitudes. Indeed, in the literature, model clipping has demonstrated its effectiveness in mitigating the impact of malicious model updates~\cite{zhang2019gradient, sparsefed, signguard}. In addition, the result also shows the importance of reducing the gradient variance of stochastic gradient and local heterogeneity to enhance robustness performance. Our work is orthogonal to existing variance or divergence reduction methods~\cite{gorbunov2022variance, malinovsky2022variance} and can be combined with them to further improve the robustness. 
We argue that AlignIns enjoys comparable robustness with several classical defense methods, for example, non-filtering-based method RFA~\cite{geomed} ($O(1 + m / (n-2m))^2$), and filtering-based method Krum~\cite{krum} ($O(1 + m / (n-2m))$)\footnote{Results of RFA and Krum are taken from~\cite{kappa}. Note that the definition of $\kappa$ in~\cite{kappa} is different from ours, but the difference part can be reduced to a constant bound. Therefore, we can safely incorporate these results into our discussion without losing generality.}. 

\end{remark}
\end{lemma}

\subsection{Propagation Error of AlignIns in FL}
Based on the $\kappa$-robustness of AlignIns, we analyze its \textit{propagation error} during training. Specifically, let $\theta$ denote the model trained with Fed-AlignIns under backdoor attacks, where $m$ of the $n$ clients are malicious, and let $\theta^{*}$ denote a model trained exclusively with benign clients using FedAvg. Starting from the same initial model $\theta^0$, we aim to measure the difference between these two models after $T$ rounds of training, defined as
$
\| \theta^T - \theta^{T, *} \|,
$
referred to as the propagation error~\cite{sparsefed}. Let $\theta^{t, +}$ represent the output of AlignIns at the $t$-th round. If the highest level of robustness is not achieved at round $t$, the error
$
\|  \theta^t - \theta^{t, +} \|
$
will propagate to the next round, resulting in a shifted starting point for local SGD at round $t+1$. This discrepancy will gradually widen the gap between $\theta$ and $\theta^{*}$. Our analysis captures this robustness error at each round and examines its cumulative effect after $T$ rounds. In \autoref{lemma:AlignInscertified_radius}, we show that, assuming \autoref{ass1}--\ref{ass3} hold, the propagation error of AlignIns remains bounded.

\begin{lemma}[Bounded Propagation Error] \label{lemma:AlignInscertified_radius}
Let \autoref{ass1}--\ref{ass3} hold. If the local learning rate $\eta \leq 1/2\tau$, the propagation error of AlignIns is bounded as 
    \begin{align*}
        \| \theta^T - \theta^{T, *} \| \leq \phi(T)(2 + 3\mu^2 )^{\phi(T) }(\kappa + 2 \bar{\nu}),
    \end{align*}
    under backdoor attacks where $m$ out of $n$ clients are malicious. Here, 
    $\kappa$ is given in \autoref{lemma:AlignInskappa}, $\phi(T) = \sum^T_{t=1} (\alpha^t)^2$ is the cumulative global learning rate, and $\alpha^t$ is a global learning rate scheduler, possibly static.
    \begin{proof}
    The detailed proof is in \apxautoref{apdx: proof_of_pro_error}.
\end{proof}

\begin{remark}
When $T \rightarrow \infty$, $\phi(T)$ converges to a constant for learning rate schedulers like exponential decay, which implies a constant bounded on propagation error. The result shows that besides the robustness error bounded by $\kappa$, the error of local gradient estimation, which is bounded by $\bar{\nu}$, in local SGD also propagates during the training, increasing the overall propagation error. This is because at any round $t$, if the benign starting point for local training is the same, \ie, $\theta^t=\theta^{t, *}$, then the local gradients/model updates on $\theta^t$ and $\theta^{t, *}$ will be identical for benign clients. Therefore, the gap between the updated global models $\theta^{t+1}$ and $\theta^{t+1, *}$ solely depends on the robustness error (\ie, the effectiveness of AlignIns in filtering out malicious updates). However, if $\theta^t\neq\theta^{t, *}$, which means $\theta^t$ is not benign and has been poisoned in previous rounds, the local gradients/model updates on $\theta^t$ and $\theta^{t, *}$ will differ for benign clients, resulting in an error bounded by the gradient variance, even if AlignIns successfully filters out all malicious updates. Hence, to further reduce the propagation error, AlignIns can be combined with variance-reduction methods like~\cite{gorbunov2022variance, malinovsky2022variance}, which is orthogonal to AlignIns.
\end{remark}
\end{lemma}

\section{Experimental Settings}
\textit{Datasets: } 
In our experiments, we primarily use  CIFAR-10~\cite{cifar10_100} and CIFAR-100~\cite{cifar10_100} datasets to evaluate the performance of various defense methods. Additionally, we present the superior performance of AlignIns on other benchmark datasets (MNIST~\cite{mnist}, FMNIST~\cite{fmnist}, and Sentiment140~\cite{sentiment140}) in \apxautoref{apdx: more_dataset}. For all datasets, we simulate a cross-silo FL system with $20$ clients. \textit{Additionally, we also present the superior performance of AlignIns on a cross-device FL system with $100$ clients and client sampling.} We consider both IID and non-IID settings. For IID settings, we distribute the training data evenly to local clients. 
For non-IID settings, we follow~\cite{hsu2019measuring, lockdown, mm} to use \textit{Dirichlet distribution} $Dir(\beta)$ to simulate the non-IID settings with a default non-IID degree $\beta=0.5$.

\textit{Learning Settings: } 
We use SGD as the local solver, with the initial learning rates set as $\alpha=1.0$ and $\eta=0.1$ and the number of local training epochs set as $2$. The number of training rounds is set as $T=100$ for CIFAR-100 and $T=150$ for CIFAR-10. For AlignIns, the default filtering radii are set as $\lambda_c = 1.0$ and $\lambda_s = 1.0$. We conduct extensive experiments to study the impact of filtering radii and present results and analysis in \apxautoref{apdx:impact_of_filter_radii}.
 The default masking parameter is set as $k=0.3 \times d$, where $d$ is the model dimension so that the Top-$30$\% of model parameters are used for the $\mathrm{MPSA}$ checking.

\textit{Evaluated Attack Methods: }
We consider $5$ SOTA backdoor attacks, including \textit{Badnet}~\cite{badnet}, \textit{DBA}~\cite{dba}, \textit{Scaling}~\cite{scaling}, \textit{PGD}~\cite{PGD}, and \textit{Neurotoxin}~\cite{Neurotoxin}. We provide the detailed attack model and settings for attack methods in \apxautoref{apdx: attack_model}--\ref{apdx: attack_setting}. 
We present the empirical performance of AlignIns under the strong \textit{trigger-optimization attack}~\cite{f3ba} in \apxautoref{apdx: triggeroptimization}. Moreover, we study the potential \textit{adaptive attacks} tailored to AlignIns and \textit{untargeted attacks}~\cite{signguard} in \apxautoref{apdx: adaptive_attack}--\ref{apdx: untargeted_attack}, although these are beyond the primary scope of this work.
To simulate effective backdoor attacks (achieving a BA over $60\%$~\cite{mesas}), the malicious client will poison $r=50\%$ of its local data, where $r$ represents the \textit{data poisoning ratio}. The \textit{attack ratio} is set to $20$\% by default, which means $20$\% of the clients in the system are malicious. Experiments of AlignIns on defending backdoor attacks with various attack ratios are given in \apxautoref{apdx: attack_ratio}.

\begin{table*}[htbp]
  \centering
  \caption{The clean MA, BA, and RA results of baselines and AlignIns on IID CIFAR-10 and CIFAR-100 datasets. Results are shown in $\%$.}
  \label{tab:main_results}
  \renewcommand{\arraystretch}{1.1}
  \scalebox{0.78}{
    \begin{tabular}{c|c|c|cccc|cccc|cccc|cc}
    \toprule
    \multirow{3}[6]{*}{\textbf{\makecell*[c]{Dataset \\ (Model)}}} & \multirow{3}[6]{*}{\textbf{Methods}} & \multirow{3}[6]{*}{\textbf{\makecell*[c]{Clean \\ MA$\uparrow$}}} & \multicolumn{4}{c|}{\textbf{Badnet}}     & \multicolumn{4}{c|}{\textbf{DBA}}     & \multicolumn{4}{c|}{\textbf{Neurotoxin}} & \multirow{3}[6]{*}{\textbf{\makecell*[c]{Avg. \\ BA$\downarrow$ }}} & \multirow{3}[6]{*}{\textbf{\makecell*[c]{Avg. \\ RA$\uparrow$ }}}\\
\cmidrule(r){4-7} \cmidrule(r){8-11} \cmidrule(r){12-15}         &       &       & \multicolumn{2}{c}{BA$\downarrow$} & \multicolumn{2}{c|}{RA$\uparrow$} & \multicolumn{2}{c}{BA$\downarrow$} & \multicolumn{2}{c|}{RA$\uparrow$} & \multicolumn{2}{c}{BA$\downarrow$} & \multicolumn{2}{c|}{RA$\uparrow$} \\
\cmidrule(r){4-5} \cmidrule(r){6-7} \cmidrule(r){8-9} \cmidrule(r){10-11} \cmidrule(r){12-13}  \cmidrule(r){14-15}        &       &       & $r$=0.3 & $r$=0.5 & $r$=0.3 & $r$=0.5 & $r$=0.3 & $r$=0.5 & $r$=0.3 & $r$=0.5 & $r$=0.3 & $r$=0.5 & $r$=0.3 & $r$=0.5  \\
    \midrule
    \multirow{9}[4]{*}{\rotatebox{90}{\makecell*[c]{CIFAR-10 \\ (ResNet9~\cite{he2016deep})}}} & FedAvg & 89.47 & 51.56 & 67.61 & 45.79 & 31.24 & 56.21 & 70.42 & 40.62 & 27.92 & 44.89 & 79.40 & 50.41 & 19.60 & 61.68 & 35.93 \\
          & FedAvg* & 89.47 & 2.06  & 2.06  & 85.60 & 85.60 & 2.06  & 2.06  & 85.60 & 85.60 & 2.06  & 2.06  & 85.60 & 85.60 & 2.06 & 85.60\\
\cmidrule{2-17}          & RLR   & 79.16 & \underline{2.32}  & \textbf{2.00} & 76.72 & 73.33 & 3.01  & 3.04  & 77.09 & 77.13 & 3.12  & 3.87  & 73.98 & 73.29 & \underline{2.89} & 35.93\\
          & RFA   & 87.73 & 70.67 & 90.24 & 27.74 & 9.26  & 47.67 & 66.97 & 47.29 & 30.14 & 81.27 & 96.13 & 17.11 & 3.69 & 75.49 & 22.54\\
          & MKrum & 87.02 & 81.10 & 97.47 & 18.11 & 2.51  & \underline{2.17}  & \underline{4.33}  & \underline{83.89} & \underline{79.10} & 65.28 & 89.18 & 31.81 & 10.01 & 56.59 & 37.57 \\
          & Foolsgold & 89.49 & 69.14 & 68.84 & 29.64 & 30.10 & 51.18 & 60.73 & 44.83 & 36.08 & \underline{2.91} & \underline{2.82}  & \underline{85.27} & \underline{84.76} & 42.60 & 51.78\\
          & MM    & 89.15 & 41.19 & 93.88 & 53.88 & 6.01  & 52.24 & 51.30 & 43.54 & 45.08 & 43.92 & 83.92 & 51.12 & 15.11 & 61.08 & 35.79\\
          & Lockdown & 88.56 & 6.31  & 10.82 & \underline{81.88} & \underline{79.50 }& 11.63 & 6.03  & 78.82 & 75.77 & 3.40  & 3.27  & 82.73 & 83.14 & 6.91 & \underline{80.31}\\
          & \cellcolor[rgb]{ .816,  .808,  .808}AlignIns & \cellcolor[rgb]{ .816,  .808,  .808}88.64 & \cellcolor[rgb]{ .816,  .808,  .808}\textbf{1.91} & \cellcolor[rgb]{ .816,  .808,  .808}\underline{2.21} & \cellcolor[rgb]{ .816,  .808,  .808}\textbf{86.03} & \cellcolor[rgb]{ .816,  .808,  .808}\textbf{85.57} & \cellcolor[rgb]{ .816,  .808,  .808}\textbf{2.13} & \cellcolor[rgb]{ .816,  .808,  .808}\textbf{2.14} & \cellcolor[rgb]{ .816,  .808,  .808}\textbf{85.77} & \cellcolor[rgb]{ .816,  .808,  .808}\textbf{85.88} & \cellcolor[rgb]{ .816,  .808,  .808}\textbf{2.66} & \cellcolor[rgb]{ .816,  .808,  .808}\textbf{2.20} & \cellcolor[rgb]{ .816,  .808,  .808}\textbf{85.46} & \cellcolor[rgb]{ .816,  .808,  .808}\textbf{85.31} & \cellcolor[rgb]{ .816,  .808,  .808}\textbf{2.21} & \cellcolor[rgb]{ .816,  .808,  .808}\textbf{85.67}\\
    \midrule
       
         \multicolumn{1}{c@{}}{} & \multicolumn{1}{c}{} & \multicolumn{1}{c}{} &  &  &  & \multicolumn{1}{c}{} &  &  &  & \multicolumn{1}{c}{} &  & & & \multicolumn{1}{c}{} & &  \vspace{-14pt} \\
    \midrule
    \multirow{9}[4]{*}{\rotatebox{90}{\makecell*[c]{CIFAR-100 \\ (VGG9~\cite{simonyan2014very})}}} & FedAvg & 64.29 & 99.20 & 99.54 & 0.68  & 0.35  & 99.25 & 99.36 & 0.64  & 0.54  & 94.41 & 93.36 & 4.36  & 5.28 & 97.52 & 1.98\\
          & FedAvg* & 64.29 & 0.62  & 0.62  & 53.03 & 53.03 & 0.62  & 0.62  & 53.03 & 53.03 & 0.62  & 0.62  & 53.03 & 53.03 & 0.62 & 53.03\\
\cmidrule{2-17}          & RLR   & 44.34 & 96.57 & 99.85 & 1.81  & 0.12  & 24.41 & 94.08 & 24.97 & 3.22  & \textbf{0.04} & \textbf{0.00} & 29.07 & 29.73 & 52.49 & 14.82\\
          & RFA   & 53.92 & 4.32  & \underline{1.45}  & 37.60 & 39.88 & 2.15  & \underline{0.78}  & \underline{39.73} & \underline{41.51} & 99.74 & 89.59 & 0.21  & 6.59 & 33.01 & 27.59\\
          & MKrum & 51.28 & \underline{1.33}  & 1.54  & \underline{38.13} & 38.49 & \underline{1.36}  & 1.54  & 37.85 & 37.91 & 99.82 & 99.87 & 0.12  & 0.10 & 36.21 & 25.49\\
          & Foolsgold & 64.13 & 99.02 & 99.30 & 0.83  & 0.57  & 99.15 & 99.39 & 0.74  & 0.51  & 21.79 & 6.21  & 42.06 & 46.40 & 70.81 & 15.19\\
          & MM    & 63.26 & 99.51 & 99.87 & 0.37  & 0.11  & 99.53 & 99.70 & 0.35  & 0.19  & 98.48 & 98.97 & 1.32  & 0.83 & 99.34 & 0.53\\
          & Lockdown & 62.88 & 55.21 & 24.14 & 28.45 & \underline{43.06} & 34.37 & 49.02 & 34.06 & 27.93 & 0.85  & 0.67  & \underline{42.66} & \underline{47.04} & \underline{27.38} & \underline{37.20}\\
          & \cellcolor[rgb]{ .816,  .808,  .808}AlignIns & \cellcolor[rgb]{ .816,  .808,  .808}63.45 & \cellcolor[rgb]{ .816,  .808,  .808}\textbf{0.79} & \cellcolor[rgb]{ .816,  .808,  .808}\textbf{0.71} & \cellcolor[rgb]{ .816,  .808,  .808}\textbf{50.45} & \cellcolor[rgb]{ .816,  .808,  .808}\textbf{51.53} & \cellcolor[rgb]{ .816,  .808,  .808}\textbf{0.45} & \cellcolor[rgb]{ .816,  .808,  .808}\textbf{0.57} & \cellcolor[rgb]{ .816,  .808,  .808}\textbf{50.81} & \cellcolor[rgb]{ .816,  .808,  .808}\textbf{52.08} & \cellcolor[rgb]{ .816,  .808,  .808}\underline{0.49} & \cellcolor[rgb]{ .816,  .808,  .808}\underline{0.53} & \cellcolor[rgb]{ .816,  .808,  .808}\textbf{51.11} & \cellcolor[rgb]{ .816,  .808,  .808}\textbf{50.66} & \cellcolor[rgb]{ .816,  .808,  .808}\textbf{0.59} & \cellcolor[rgb]{ .816,  .808,  .808}\textbf{51.11} \\
    \bottomrule
    \end{tabular}%
    }
    \vspace{-10pt}
  \label{tab:addlabel}%
\end{table*}%

\textit{Evaluated Defense Methods: }
We present the detailed defense model in \apxautoref{apdx: defense_model}. We comprehensively compare AlignIns with the non-robust baseline FedAvg and six existing SOTA defense methods, including \textit{RLR}~\cite{rlr}, \textit{RFA}~\cite{geomed}, \textit{Multi-Krum (MKrum)}~\cite{krum}, \textit{Foolsgold}~\cite{Foolsgold}, \textit{Multi-Metric (MM)}~\cite{mm}, and \textit{Lockdown}~\cite{lockdown}.
Additionally, we compare our approach with an ideally perfect filtering-based robust aggregation, \textit{FedAvg*}, which is assumed to perfectly identify and remove all malicious updates and average all the benign updates to update the global model.

\textit{Evaluation Metrics: }
We use three metrics to evaluate the performance of defense methods, including \textbf{main task accuracy (MA)}, which measures the percentage of clean test samples that are accurately classified to their ground truth labels by the global model; \textbf{backdoor attack accuracy (BA)}, which measures the percentage of triggered samples that are misclassified to the target label by the global model; and \textbf{robustness accuracy (RA)}, which measures the percentage of triggered samples that are accurately classified to their ground-truth labels by the global model, despite the presence of the trigger. A good defense method should achieve high MA and RA and low BA.

\section{Experimental Results}
\textbf{Main results in IID setting.} In \autoref{tab:main_results}, we report the performance of various defense methods under no attack (denoted by ``Clean''), Badnet, DBA, and Neurotoxin attacks for IID CIFAR-10 and CIFAR-100. The best results are highlighted in \textbf{bold font}, and the second best results are \underline{underlined}. Overall, \textbf{\textit{AlignIns demonstrates superior performance compared with other baselines as it achieves the best average BA and RA over three attack methods.}} Specifically, for CIFAR-10, while RLR offers a satisfactory degree of robustness (an average BA of $2.89\%$), it suffers from a notable decline in RA, with an average reduction of $49.74\%$ in comparison to AlignIns. This drop results from RLR's strategy of flipping the global learning rate for parameters in the aggregated model update that are inconsistent with the majority's sign, consequently resulting in the loss of benign local parameters. AlignIns, however, demonstrates outstanding performance with consistently low BA and high RA, ranking first or second among its counterparts. Notably, compared to the second-best results, AlignIns achieves an average improvement of $+0.68\%$ in BA and $+5.36\%$ in RA.
Similarly, superior results are observed in CIFAR-100 experiments, where AlignIns significantly outperforms other methods in both BA and RA. These results underscore AlignIns' effectiveness as a promising defense method for protecting FL from various backdoor attacks, significantly enhancing the trustworthiness of FL systems.

\begin{figure}[t]
    \centering
    \includegraphics[width=0.90\linewidth]{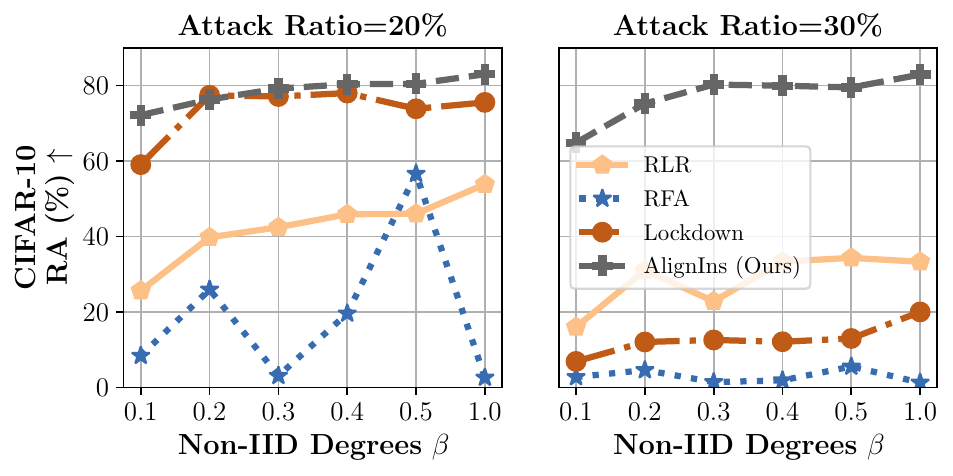}
    \vspace{-5pt}
    \caption{RA of AlignIns under various non-IID degrees, compared with Lockdown, RFA, and RLR under Neurotoxin .}
    \label{fig:noniid}
    \vspace{-5pt}
\end{figure}
\textbf{Effectiveness under various Non-IID degrees.}
We examine the defense performance of AlignIns across various degrees of non-IIDness, a factor that significantly complicates backdoor defense. \autoref{fig:noniid} presents the RA of AlignIns under different non-IID conditions on the CIFAR-10 dataset, compared with Lockdown, RFA, and RLR. The experiments were conducted using the Neurotoxin attack, with both a default attack ratio of $20\%$ and a higher attack ratio of $30\%$.
The Dirichlet parameter $\beta$ varies from $0.1$ to $1.0$, where a smaller $\beta$ suggests a more intense non-IIDness. We observe that only AlignIns consistently attains robustness against strong Neurotoxin attacks with a varying $\beta$. 
Specifically, as $\beta$ increases, the RA of AlignIns, Lockdown, and RLR increases correspondingly. However, AlignIns outperforms them with a consistently higher RA. When the attack ratio rises to $30\%$, RLR, RFA, and Lockdown fail to provide satisfactory robustness. However, our method AlignIns still demonstrates its robustness under various non-IIDness, even in an extremely non-IID case when $\beta=0.1$. AlignIns is designed to examine the alignment of model updates on important parameters only, hence, it mitigates the challenge of identifying malicious model updates in non-IID settings where updates are heterogeneous, thereby achieving superior performance in even extreme non-IID settings compared with existing methods. We also provide more comprehensive results of AlignIns and other baselines on non-IID datasets in \apxautoref{apdx: comprehensive_result_on_non_iid_dataset}.

\textbf{Effectiveness of AlignIns in cross-device FL with client sampling.}
\label{cross_device_exp}
While most of our experiments focus on the cross-silo FL setting, evaluating the cross-device FL scenario is also essential given the large number of clients involved. For this purpose, we simulate a cross-device FL environment with $100$ clients, where the server randomly selects $20$ clients per round for training. We conduct experiments on IID and non-IID CIFAR-10 cases using Foolsgold, Lockdown, RLR, and AlignIns and summarize the MA, BA, and RA results in \autoref{tab: cross_device}. The results show that both Foolsgold and Lockdown completely lose their effectiveness in both cases, achieving an average RA of nearly $0.00\%$. RLR achieves a moderate level of backdoor robustness but at the cost of main task accuracy, with an average MA of only $49.19\%$. In contrast, AlignIns performs robustly in the cross-device FL setting, achieving a significantly lower BA in both IID ($0.92\%$) and non-IID ($1.90\%$) cases compared with other methods. Furthermore, AlignIns achieves an average RA of $79.28\%$. These results highlight AlignIns's ability to maintain both accuracy and robustness in challenging cross-device FL scenarios, underscoring its adaptability and effectiveness in real-world applications.

\begin{table}[t]
  \centering
  \caption{Performance of different methods in cross-device FL settings on IID and non-IID CIFAR-10 datasets under Badnet attack.}
  \scalebox{0.74}{
    \begin{tabular}{c|cccccc|c}
    \toprule
    \multirow{2}[4]{*}{\textbf{Method}} & \multicolumn{3}{c}{\textbf{CIFAR-10 (IID)}} & \multicolumn{3}{c|}{\textbf{CIFAR-10 (Non-IID)}} & \multirow{2}[2]{*}{\textbf{\makecell*[c]{Avg. \\ RA$\uparrow$}}}\\
    \cmidrule(r){2-4} \cmidrule(r){5-7} 
         & MA$\uparrow$    & BA$\downarrow$    & RA$\uparrow$    & MA$\uparrow$    & BA$\downarrow$    & RA$\uparrow$ \\
    \midrule
    Foolsgold &   82.99  &  99.99   &  0.01    &   67.97 & 99.99  &  0.00 &  0.01 \\
    Lockdown &   \underline{83.52}  &  99.99   &  0.00    &   \underline{73.91} & 99.92  &  0.06 &  0.00 \\
    RLR &   56.81 & \underline{4.67} & \underline{55.38} & 41.56 & \underline{14.12} & \underline{38.17} & \underline{46.78}\\
    \rowcolor[rgb]{ .816,  .808,  .808}AlignIns &   \textbf{85.01}  &  \textbf{0.92}   &   \textbf{82.74}   &  \textbf{79.51}   & \textbf{1.90}  &  \textbf{75.81} &  \textbf{79.28} \\
    \bottomrule
    \end{tabular}%
}
    \vspace{-5pt}
  \label{tab: cross_device}%
\end{table}%

\textbf{Ablation study of AlignIns.}
\label{sec:ablation}
As AlignIns consists of two alignment components ($\mathrm{TDA}$ and $\mathrm{MPSA}$) to improve backdoor robustness, we conduct a detailed ablation study to investigate how each component functions. Experimental results on IID and non-IID CIFAR-10 datasets under Badnet attack are summarized in \autoref{tab:ablation}. \textbf{\textit{(i) Component ablation.}} We observe that using $\mathrm{MPSA}$ or $\mathrm{TDA}$ alone in IID scenarios only slightly reduces robustness compared to AlignIns, as benign updates follow consistent patterns that enable effective detection by a single metric. In non-IID settings, however, where local updates diverge, neither $\mathrm{MPSA}$ nor $\mathrm{TDA}$ alone provides sufficient robustness. When combined, $\mathrm{MPSA}$ and $\mathrm{TDA}$ improve BA and RA from $94.07\%$ and $5.79\%$ to $47.04\%$ and $45.30\%$, respectively, showing their complementary strengths. AlignIns further enhances robustness by integrating $\mathrm{MPSA}$, $\mathrm{TDA}$, and post-filtering model clipping, which normalizes benign update magnitudes and improves malicious update detection, yielding the highest average RA. \textbf{\textit{(ii) Masking parameter $k$ ablation.}} We try to involve more non-essential parameters in the $\mathrm{MPSA}$ checking by using the Top-$50$\%$/70$\% of parameters to calculate $\mathrm{MPSA}$ values. By doing so, the effectiveness of malicious identification is reduced. In contrast, when using the Top-$30$\% of parameters, compared to the Top-$50$\% case, BA and RA are improved by $+30.89\%$ and $+25.34\%$, respectively. This demonstrates the effectiveness of focusing important parameters when calculating $\mathrm{MPSA}$ in improving the filtering accuracy, especially in non-IID cases. \textbf{\textit{(iii) Variance reduction method further enhances robustness.}} Our theoretical results reveal the impact of variance reduction techniques on improving the robustness of AlignIns and reducing the propagation error of AlignIns in FL, we additionally test a variant of AlignIns named ``AlignIns$^+$'', in which local SGD with momentum is used to reduce the local gradient variance with momentum coefficient $0.1$. AlignIns$^+$ achieves a slightly better performance than AlignIns, verifying our theoretical results.

\begin{table}[t]
  \centering
  \caption{Performance of different components in AlignIns.}
  \scalebox{0.70}{
    \begin{tabular}{c|ccc|ccc|c}
    \toprule
    \multirow{2}[4]{*}{\textbf{Configuration}} & \multicolumn{3}{c}{\textbf{CIFAR-10 (IID)}} & \multicolumn{3}{c|}{\textbf{CIFAR-10 (non-IID)}} & \multirow{2}[2]{*}{\textbf{\makecell*[c]{Avg. \\ RA$\uparrow$}}} \\
\cmidrule(r){2-4} \cmidrule(r){5-7}          & MA$\uparrow$    & BA$\downarrow$    & RA$\uparrow$    & MA$\uparrow$    & BA$\downarrow$    & RA$\uparrow$ \\
    \midrule
    $\mathrm{MPSA}$($30\%$) & \underline{88.55} & 2.88  & 85.02 & 80.65 & 94.07 & 5.79 & 45.41 \\
    $\mathrm{TDA}$   & \textbf{88.56} & 3.82  & 83.88 & \underline{83.86} & 77.58 & 21.31 & 52.60 \\
    $\mathrm{MPSA}$($70\%$+$\mathrm{TDA}$ & 88.14 & \underline{2.18}  & \underline{85.77} & 83.84 & 61.83 & 31.86 & 58.82\\
    $\mathrm{MPSA}$($50\%$)+$\mathrm{TDA}$ & 88.05 & 2.21  & 85.46 & \textbf{84.12} & 77.93 & 19.96 & 52.71 \\
    $\mathrm{MPSA}$($30\%$)+$\mathrm{TDA}$ & 88.14 & \textbf{2.04} & \textbf{85.82} & 83.65 & \underline{47.04} & \underline{45.30} & \underline{65.56} \\
    \rowcolor[rgb]{ .816,  .808,  .808}AlignIns  & 88.05 & 2.44  & 85.27 & 82.88 & \textbf{1.70} & \textbf{81.32} & \textbf{83.30} \\
    \midrule
    AlignIns$^+$  & 88.48 & 2.14  & 85.74 & 83.31 & 1.11 & 82.13 & 83.94\\
    \bottomrule
    \end{tabular}%
    }
    \vspace{-1pt}
  \label{tab:ablation}%
\end{table}%

\section{Conclusion}
This paper introduces a novel defense method AlignIns to defend against backdoor attacks in FL. AlignIns examines each model update's direction at different granularity levels, thus effectively identifying stealthy malicious local model updates and filtering them out to avoid them participating in aggregation in FL to enhance robustness. We provide a theoretical analysis of AlignIns' robustness and its impact on propagation errors in FL. Extensive experiments demonstrate the effectiveness of AlignIns, with results showing that it outperforms SOTA defense methods against various advanced attacks.

\clearpage
{
\small
    \bibliographystyle{ieeenat_fullname}
    \bibliography{main}}

\clearpage
\maketitlesupplementary

\section{Attack Model and Detailed Attack Settings}

\subsection{Attack Model} 
\label{apdx: attack_model}
We follow the threat model in previous works~\cite{deepsight, krum, dnc}. Specifically, the attacker controls $m$ malicious clients, which can be fake injected into the system by the attacker or benign clients compromised by the attacker. These malicious clients are allowed to co-exist in the FL system. \textbf{\textit{i) Attacker's goal.}}
The backdoor attackers in FL have two primary objectives. First, they aim to maintain the accuracy of the global model on benign inputs, ensuring that its overall performance remains unaffected. Second, they seek to manipulate the global model so that it behaves as predefined by the attacker on inputs containing a specific trigger, such as misclassifying triggered inputs to a specific backdoor label.  \textbf{\textit{ii) Attacker's capability.}}
The attacker controls $m$ malicious clients in FL. We consider three levels of the attacker's capability in manipulating their model updates, including \textit{weak level}, \textit{median level}, and \textit{strong level}. The malicious clients controlled by weak attackers (e.g., Badnet~\cite{badnet} and DBA~\cite{dba}) are only able to manipulate their local datasets to generate malicious local model updates and send them to the server for aggregation. For a median attacker, malicious clients can additionally modify the training algorithm (e.g., Scaling~\cite{scaling} and PGD~\cite{PGD}) to generate malicious local model updates. These two assumptions are common in existing works for attackers who control malicious devices but do not have access to additional information from servers or benign clients. For a strong attacker (e.g., Neurotoxin~\cite{Neurotoxin}), it can access and leverage the global information from the server to improve the attack. Note that the defense method employed by the server is confidential to the attacker.

\begin{figure}
    \centering
    \includegraphics[width=\linewidth]{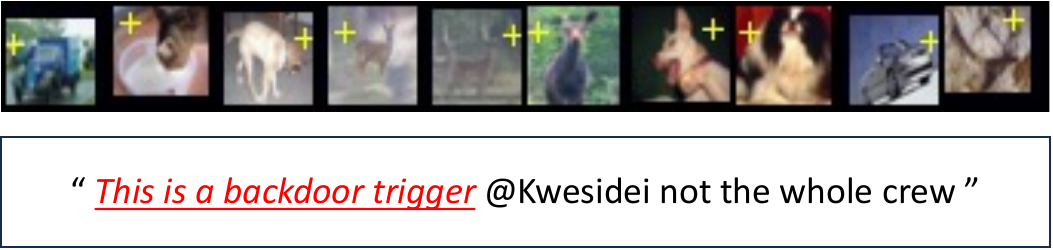}
    \caption{Illustration of backdoor triggers used in evaluation.}
    \label{fig:backdoor}
\end{figure}

\subsection{More Detailed Settings of Attack Methods.}
\label{apdx: attack_setting}
For image datasets, we add a ``plus'' trigger to benign samples to generate the poisoned data samples. For Sentiment140 dataset, we insert a trigger sentence ``This is a backdoor trigger'' into benign samples to generate poisoned data samples. The example of triggered data samples in CIFAR-10 and Sentiment140 are shown in \autoref{fig:backdoor}. For DBA attack, we decompose the "plus" trigger into four local patterns, and each malicious client only uses one of these local patterns. For Scaling attack, we use a scale factor of $2.0$ to scale up all malicious model updates. For PGD attack, malicious local models are projected onto a sphere with a radius equal to the $L_2$-norm of the global model in the current round for all datasets, except CIFAR-10 where we make the radius of the sphere be 10 times smaller than the norm. For Neurotoxin attack, malicious model updates are projected to the dimensions that have Bottom-$75\%$ importance in the aggregated update from the previous round.

\section{Defense Model}
\label{apdx: defense_model}
In this work, we assume the server to be the defender. \textbf{\textit{i) Defender's goal.}}
As stated in~\cite{fltrust}, an ideal {defense method} against poisoning attacks in FL should consider the following three aspects: \textit{Fidelity}, \textit{Robustness}, and \textit{Efficiency}. To ensure fidelity, the {defense method} does not significantly degrade the global model's performance on benign inputs, thus preserving its effectiveness. For robustness, the {defense method} should successfully mitigate the impact of malicious model updates, limiting the global model's malicious behavior on triggered inputs. Regarding efficiency, the {defense method} should be computationally efficient, ensuring that it does not hinder the overall efficiency of the training process. \textit{In this work, we assume that the server aims to achieve the highest level of robustness by removing all malicious updates without significant computational complexity and accuracy degradation on benign inputs.} \textbf{\textit{ii) Defender's capability.}}
In FL, the server has no access to the local datasets of clients, but it has the global model and all the local model updates. We assume the server has no prior knowledge of the number of malicious clients. We also assume that each client transmits their local update anonymously, making the actions of individual clients untraceable. Additionally, the server does not know the specifics of backdoor attacks, such as the type of trigger involved. To defend against backdoor attacks, the server will apply a robust aggregation rule $F$ to the local model updates received from clients and generate an aggregated model update at each training round. 

\section{More Superior Results of AlignIns}

\subsection{Comprehensive Results on non-IID Datasets}
\label{apdx: comprehensive_result_on_non_iid_dataset}

\begin{table*}[htbp]
  \centering
  \caption{The MA, BA, and RA results of baselines and AlignIns on non-IID CIFAR-10 and CIFAR-100 datasets. Results are shown in $\%$.}
  \label{tab:main_results_extra}
  \renewcommand{\arraystretch}{1.1}
  \scalebox{0.78}{
    \begin{tabular}{c|c|c|cccc|cccc|cccc|cc}
      \toprule
      \multirow{3}[6]{*}{\textbf{\makecell*[c]{Dataset \\ (Model)}}} & \multirow{3}[6]{*}{\textbf{Methods}} & \multirow{3}[6]{*}{\textbf{\makecell*[c]{Clean \\ MA$\uparrow$}}} & \multicolumn{4}{c|}{\textbf{Badnet}} & \multicolumn{4}{c|}{\textbf{DBA}} & \multicolumn{4}{c|}{\textbf{Neurotoxin}} & \multirow{3}[6]{*}{\textbf{\makecell*[c]{Avg. \\ BA$\downarrow$}}} & \multirow{3}[6]{*}{\textbf{\makecell*[c]{Avg. \\ RA$\uparrow$}}} \\
      \cmidrule(r){4-7} \cmidrule(r){8-11} \cmidrule(r){12-15}
      & & & \multicolumn{2}{c}{BA$\downarrow$} & \multicolumn{2}{c|}{RA$\uparrow$} & \multicolumn{2}{c}{BA$\downarrow$} & \multicolumn{2}{c|}{RA$\uparrow$} & \multicolumn{2}{c}{BA$\downarrow$} & \multicolumn{2}{c|}{RA$\uparrow$} \\
      \cmidrule(r){4-5} \cmidrule(r){6-7} \cmidrule(r){8-9} \cmidrule(r){10-11} \cmidrule(r){12-13} \cmidrule(r){14-15}
      & & & $r$=0.3 & $r$=0.5 & $r$=0.3 & $r$=0.5 & $r$=0.3 & $r$=0.5 & $r$=0.3 & $r$=0.5 & $r$=0.3 & $r$=0.5 & $r$=0.3 & $r$=0.5 \\
      \midrule
      \multirow{9}[4]{*}{\rotatebox{90}{\makecell*[c]{CIFAR-10 \\ (ResNet9)}}} 
      & FedAvg & 85.05 & 42.34 & 86.33 & 51.60 & 13.22 & 42.24 & 71.64 & 49.63 & 25.26 & 42.29 & 76.63 & 48.76 & 20.73 & 53.57 & 36.29 \\
      & FedAvg* & 85.05 & 1.78 & 1.78 & 83.09 & 83.09 & 1.78 & 1.78 & 83.09 & 83.09 & 1.78 & 1.78 & 83.09 & 83.09 & 1.78 & 83.09 \\
      \cmidrule{2-17}
      & RLR & 59.87 & \underline{3.27} & \textbf{0.94} & 55.54 & 55.53 & \underline{1.98} & \underline{1.87} & 59.98 & 59.52 & \textbf{0.21} & \textbf{0.27} & 45.60 & 46.02 & \underline{1.92} & 53.04 \\
      & RFA & 79.80 & 56.26 & 97.42 & 36.49 & 2.30 & 53.70 & 90.70 & 39.00 & 8.10 & 4.29 & 22.26 & 71.93 & 56.60 & 50.27 & 39.36 \\
      & MKrum & 70.89 & 72.70 & 95.57 & 20.98 & 3.71 & 2.12 & 53.81 & 69.80 & 35.09 & \underline{1.18} & \underline{1.22} & 74.02 & 71.08 & 49.58 & 37.78 \\
      & Foolsgold & 85.97 & 20.24 & 83.27 & 68.91 & 16.14 & 42.20 & 63.56 & 50.79 & 31.62 & 3.77 & 1.49 & \underline{78.08} & \underline{80.22} & 42.88 & 62.45 \\
      & MM & 82.02 & 50.52 & 95.70 & 41.41 & 4.08 & 66.88 & 43.69 & 28.18 & 47.38 & 85.58 & 98.86 & 13.02 & 1.04 & 63.12 & 30.83 \\
      & Lockdown & 84.05 & 6.68 & 8.01 & \underline{75.23} & \underline{75.73} & 7.11 & 6.03 & \underline{76.63} & \underline{75.77} & 1.24 & 2.19 & 73.82 & 73.81 & 5.21 & \underline{75.07} \\
      & \cellcolor[rgb]{ .816, .808, .808}AlignIns & \cellcolor[rgb]{ .816, .808, .808}83.77 & \cellcolor[rgb]{ .816, .808, .808}\textbf{2.48} & \cellcolor[rgb]{ .816, .808, .808}\underline{1.7} & \cellcolor[rgb]{ .816, .808, .808}\textbf{81.17} & \cellcolor[rgb]{ .816, .808, .808}\textbf{81.32} & \cellcolor[rgb]{ .816, .808, .808}\textbf{1.54} & \cellcolor[rgb]{ .816, .808, .808}\textbf{1.10} & \cellcolor[rgb]{ .816, .808, .808}\textbf{81.24} & \cellcolor[rgb]{ .816, .808, .808}\textbf{81.11} & \cellcolor[rgb]{ .816, .808, .808}2.73 & \cellcolor[rgb]{ .816, .808, .808}2.08 & \cellcolor[rgb]{ .816, .808, .808}\textbf{81.54} & \cellcolor[rgb]{ .816, .808, .808}\textbf{80.42} & \cellcolor[rgb]{ .816, .808, .808}\textbf{1.77} & \cellcolor[rgb]{ .816, .808, .808}\textbf{80.48} \\
      \midrule
      \multicolumn{1}{c@{}}{} & \multicolumn{1}{c}{} & \multicolumn{1}{c}{} & & & & \multicolumn{1}{c}{} & & & & \multicolumn{1}{c}{} & & &  & \multicolumn{1}{c}{} & \vspace{-14pt} \\
      \midrule
      \multirow{9}[4]{*}{\rotatebox{90}{\makecell*[c]{CIFAR-100 \\ (VGG9)}}} 
      & FedAvg & 63.33 & 99.57 & 99.63 & 0.35 & 0.33 & 99.52 & 99.74 & 0.45 & 0.23 & 97.58 & 97.18 & 1.94 & 2.25 & 98.66 & 0.92 \\
      & FedAvg* & 63.33 & 0.59 & 0.59 & 50.21 & 50.21 & 0.59 & 0.59 & 50.21 & 50.21 & 0.59 & 0.59 & 50.21 & 50.21 & 0.59 & 50.21 \\
      \cmidrule{2-17}
      & RLR & 35.83 & 58.31 & 98.94 & 9.22 & 0.47 & 2.31 & 76.82 & 22.61 & 7.79 & \textbf{0.00} & 15.54 & 11.31 & 15.54 & 42.26 & 11.66 \\
      & RFA & 34.16 & \underline{3.19} & \underline{0.89} & 25.07 & 26.58 & \underline{0.91} & 4.25 & 24.68 & 25.66 & 99.47 & 8.52 & 0.36 & 22.82 & 22.51 & 20.93 \\
      & MKrum & 45.10 & 99.44 & 1.84 & 0.43 & \underline{34.89} & 99.30 & \underline{1.22} & 0.55 & \underline{34.05} & 99.71 & 99.20 & 0.23 & 0.49 & 54.69 & 14.69 \\
      & Foolsgold & 62.77 & 99.58 & 99.56 & 0.38 & 0.38 & 99.52 & 99.67 & 0.43 & 0.29 & 11.64 & 11.06 & 43.01 & 42.20 & 70.23 & 10.27 \\
      & MM & 60.22 & 99.65 & 99.93 & 0.28 & 0.04 & 99.90 & 99.94 & 0.10 & 0.06 & 99.73 & 99.82 & 0.23 & 0.14 & 99.53 & 0.18 \\
      & Lockdown & 60.91 & 29.19 & 40.08 & \underline{32.91} & 30.60 & 11.90 & 20.08 & \underline{34.97} & 32.79 & \underline{0.13} & \textbf{0.07} & \underline{44.42} & \underline{42.72} & \underline{21.73} & \underline{36.47} \\
      & \cellcolor[rgb]{ .816, .808, .808}AlignIns & \cellcolor[rgb]{ .816, .808, .808}59.18 & \cellcolor[rgb]{ .816, .808, .808}\textbf{0.66} & \cellcolor[rgb]{ .816, .808, .808}\textbf{0.54} & \cellcolor[rgb]{ .816, .808, .808}\textbf{47.51} & \cellcolor[rgb]{ .816, .808, .808}\textbf{44.67} & \cellcolor[rgb]{ .816, .808, .808}\textbf{0.19} & \cellcolor[rgb]{ .816, .808, .808}\textbf{0.42} & \cellcolor[rgb]{ .816, .808, .808}\textbf{47.33} & \cellcolor[rgb]{ .816, .808, .808}\textbf{48.77} & \cellcolor[rgb]{ .816, .808, .808}1.20 & \cellcolor[rgb]{ .816, .808, .808}\underline{1.09} & \cellcolor[rgb]{ .816, .808, .808}\textbf{49.17} & \cellcolor[rgb]{ .816, .808, .808}\textbf{45.70} & \cellcolor[rgb]{ .816, .808, .808}\textbf{0.64} & \cellcolor[rgb]{ .816, .808, .808}\textbf{47.86} \\
      \bottomrule
    \end{tabular}
  }
  \vspace{-10pt}
\end{table*}


In non-IID settings, the divergence between benign model updates will increase, thus defense methods are hard to identify malicious model updates. From \autoref{tab:main_results_extra}, We can conclude MM still fails to detect malicious model updates on two non-IID cases. Foolsgold can only exhibit a limited degree of robustness under Neurotoxin attack. Specifically, in the non-IID CIFAR-10 under DBA attack, Foolsgold was unable to effectively detect malicious model updates. This resulted in a BA of $42.20\%$ and $63.56\%$ and an RA of $50.79\%$ and $31.62\%$. The reason for this lies in the feature of the Neurotoxin attack, where the malicious model updates are projected to the Bottom-$k$ parameters of the aggregated model update in the latest round. This process makes the malicious model updates generated by Neurotoxin attacks have the same Top parameters, reducing local variance between them. Foolsgold enjoys a more accurate identification of malicious model updates as it works based on the assumption that malicious model updates are consistent with each other. In contrast, AlignIns exhibits outstanding robustness in the same case as AlignIns achieves significantly superior performance, yielding the lowest BA at $1.54\%$ and $1.10\%$, and the highest RA at $81.24\%$ and $81.11\%$. This marks an improvement of $+40.66\%$ and $+62.46\%$ in BA and $+30.45\%$ and $+49.49\%$ in RA over Foolsgold. For CIFAR-100 dataset, AlignIns still have a lower BA and higher RA than their counterparts, underlining the enhanced detection and robustness capabilities of AlignIns in challenging non-IID conditions.

\begin{table}[ht]
  \centering
  \caption{Performance of AlignIns on Tiny-ImageNet dataset.}
  \scalebox{0.8}{
    \begin{tabular}{c|cc|cc|cc}
    \toprule
    \multirow{2}[4]{*}{\textbf{Method}} & \multicolumn{2}{c}{\textbf{Badnet}} & \multicolumn{2}{c}{\textbf{Neurotoxin}} & \multirow{2}[3]{*}{\textbf{\makecell*[c]{Avg. \\ BA$\downarrow$}}} & \multirow{2}[3]{*}{\textbf{\makecell*[c]{Avg. \\ RA$\uparrow$}}}\\
\cmidrule(r){2-3} \cmidrule(r){4-5}             & BA$\downarrow$    & RA$\uparrow$      & BA$\downarrow$    & RA$\uparrow$   \\
    \midrule
    RLR  & 55.54 & 18.25 & 0.54 & 22.01 & 28.04 & 20.13 \\
    RFA  & 0.38 & 32.40 & 97.41 & 1.97 & 48.90 & 17.19 \\
    MKrum  & \underline{0.36} & \underline{32.60} & 29.37 & 25.55 & \underline{14.87} & \underline{29.08} \\
    Foolsgold  & 93.59 & 4.68 & \textbf{0.26} & \textbf{37.05} & 46.93 & 20.87 \\
    MM  & 97.01 & 2.11 & 90.85 & 5.27 & 93.93 & 3.69 \\
    Lockdown & 72.08 & 17.09 & \underline{0.34} & 28.18 & 36.21 & 22.64 \\
    \rowcolor[rgb]{ .816,  .808,  .808} AlignIns & \textbf{0.22} & \textbf{34.55} & 0.40 & \underline{36.30} & \textbf{0.31} & \textbf{35.43} \\
    \bottomrule
    \end{tabular}%
    }
  \label{tab:tiny-imagenet}%
\end{table}%

\subsection{Results on Larger Datasets}

We also evaluate AlignIns on the Tiny-ImageNet dataset, which is typically the largest dataset considered in related works. The BA and RA results are summarized in \autoref{tab:tiny-imagenet}. AlignIns demonstrates strong robustness against both BadNet and Neurotoxin attacks, achieving the lowest BA ($0.31\%$) and the highest RA ($35.43\%$). These results highlight the practical effectiveness of AlignIns on large, real-world datasets.

\subsection{Trigger-Optimization Attack}
\label{apdx: triggeroptimization}
We evaluate the experimental performance of AlignIns under the strong trigger-optimization attack. Specifically, we consider the SOTA trigger-optimization attack F3BA~\cite{f3ba} and conduct experiments on CIFAR-10 dataset under both IID and varying degrees of non-IID settings. As the results shown in \autoref{tab:f3ba}, 
FedAvg is vulnerable to F3BA as it has a high BA and low RA. Similarly, RLR also cannot provide enough robustness to F3BA especially when the data heterogeneity is high. In contrast, AlignIns consistently achieves the highest robustness across all scenarios. Specifically, compared to Bulyan, AlignIns yields an average increase of $+22.63\%$ in BA and $+19.11\%$ in RA. While trigger-optimization attacks typically search for an optimal trigger to enhance their stealthiness and effectiveness, AlignIns can still identify malicious and benign model updates by inspecting their alignments.

\begin{table}[ht]
  \centering
  \caption{Performance of AlignIns under trigger-optimization attack on CIFAR-10 dataset in both IID and non-IID settings.}
  \scalebox{0.7}{
    \begin{tabular}{c|cc|cc|cc|cc}
    \toprule
    \multirow{3}[6]{*}{\textbf{Method}} & \multicolumn{8}{c}{\textbf{Data Distritbuion}} \\
\cmidrule{2-9}          & \multicolumn{2}{c}{$\beta$=0.3} & \multicolumn{2}{c}{$\beta$=0.5} & \multicolumn{2}{c}{$\beta$=0.7} & \multicolumn{2}{c}{IID} \\
\cmidrule(r){2-3}\cmidrule(r){4-5}\cmidrule(r){6-7}\cmidrule(r){8-9}          & BA$\downarrow$    &  RA$\uparrow$   & BA$\downarrow$    & RA$\uparrow$    & BA$\downarrow$    & RA$\uparrow$    & BA$\downarrow$    & RA$\uparrow$ \\
    \midrule
    FedAvg & 93.97 &  5.13 & 93.44 & 6.06 & 94.76 & 4.83  & 94.16 & 5.50 \\
    RLR   & 92.58 &  6.71 & 93.20 &  6.42 & 81.38 &  15.80 & 86.23 & 13.23 \\
    Bulyan & \underline{60.97} & \underline{27.49}  &  \underline{8.57} & \underline{58.12} & \underline{17.82} & \underline{57.71} & \underline{15.61} & \underline{64.40} \\
    \rowcolor[rgb]{ .816,  .808,  .808}AlignIns & \textbf{5.22} &  \textbf{65.12} & \textbf{2.33} & \textbf{72.82}  & \textbf{1.99} &  \textbf{70.50} & \textbf{2.91} & \textbf{75.71} \\
    \bottomrule
    \end{tabular}%
    }
  \label{tab:f3ba}%
\end{table}%

\subsection{Effectiveness under Adaptive Attack}
\label{apdx: adaptive_attack}
Recall that in our attack model, the attacker is assumed to be unaware of the defense method the server deployed. Here, we assume the attacker has such knowledge and evaluate AlignIns under attacks tailored to circumvent it. Specifically, we design two adaptive attacks: ADA\_A, where each malicious client randomly selects a benign model update and mirrors its sign, and ADA\_B, where each malicious client aligns with the principal sign of all model updates. Results are summarized in \autoref{tab:adaptive_attack}. In the results, AlignIns shows strong resistance to both ADA\_A and ADA\_B attacks. For ADA\_A, although it leverages benign signs, $\mathrm{MPSA}$ focuses on the signs of important weights, which typically differ from those of benign models, allowing AlignIns to counter ADA\_A effectively. For ADA\_B, using the principal sign yields an $\mathrm{MPSA}$ value of $1.0$, which our $\mathrm{MZ_Score}$ can readily detect. These results confirm that AlignIns effectively limits backdoor success and preserves the main task and robust accuracy, even against adaptive attack strategies tailored to exploit its defenses.

\begin{table}[ht]
  \centering
  \caption{Performance of AlignIns on Adaptive Attacks.}
  \scalebox{0.75}{
    \begin{tabular}{c|ccc|ccc}
    \toprule
    \multirow{2}[4]{*}{\textbf{Dataset}} & \multicolumn{3}{c}{\textbf{ADA\_A}} & \multicolumn{3}{c}{\textbf{ADA\_B}} \\
    \cmidrule(r){2-4} \cmidrule(r){5-7} 
         & MA$\uparrow$    & BA$\downarrow$    & RA$\uparrow$    & MA$\uparrow$    & BA$\downarrow$    & RA$\uparrow$    \\
    \midrule
    CIFAR-10 & 88.22 & 2.34 &   85.44& 88.33 & 1.82 &  86.49  \\
    CIFAR-100 & 62.10 & 0.48 & 51.87  & 62.86 & 0.37 &  53.55  \\
    \bottomrule
    \end{tabular}%
    }
  \label{tab:adaptive_attack}%
\end{table}%

\subsection{Effectiveness under Untargeted Attack}
\label{apdx: untargeted_attack}
\begin{table}[t]
  \centering
  \caption{The MA of AlignIns under untargeted attack on CIFAR-10 dataset in both IID and non-IID settings.}
  \scalebox{0.67}{
    \begin{tabular}{c|cccc|cccc}
    \toprule
    \multirow{2}[4]{*}{\textbf{Method}} & \multicolumn{4}{c}{\textbf{Attack Ratio=}10\%}   & \multicolumn{4}{c}{\textbf{Attack Ratio=}20\%} \\
\cmidrule(r){2-5} \cmidrule(r){6-9}          & $\beta$=0.3   & $\beta$=0.5   & $\beta$=0.7   & IID   & $\beta$=0.3   & $\beta$=0.5   & $\beta$=0.7   & IID \\
    \midrule
    FedAvg &   10.95    &    13.21   &  11.66     &   20.71    & 10.85 & 12.96 & 10.33 & 18.62 \\
    RFA   & 77.43 & 78.26 & 80.45 & 87.03 & 77.02 & 76.93 & 79.76 & 86.03 \\
    MKrum & 67.99 & 71.14 & 76.76 & 86.87 & 65.61 & 74.39 & 77.16 & 86.39 \\
    SignGuard & \underline{85.11} & \underline{85.58} & \underline{86.84} & \textbf{89.23} & \textbf{85.71} & \underline{84.69} & \textbf{86.22} & \underline{88.45} \\
    \rowcolor[rgb]{ .816,  .808,  .808}AlignIns &   \textbf{85.32}    &   \textbf{85.61}    &  \textbf{87.13}     &   \textbf{89.23} & \underline{85.49} & \textbf{84.98} & \underline{86.18} & \textbf{88.54} \\
    \bottomrule
    \end{tabular}%
    }
  \label{tab:byzmean}%
  
\end{table}%

In this section, we conduct experiments to illustrate how AlignIns performs with respect to untargeted attacks (also known as Byzantine attacks). Byzantine attacks aim to degrade the model's overall performance during the training as much as possible. We consider the SOTA Byzantine attack method ByzMean~\cite{signguard} which uses the Lie attack~\cite{lie} as the backbone of the attack baseline. We also involve the SOTA Byzantine-robust method SignGuard~\cite{signguard} in our experiments. \autoref{tab:byzmean} reports the MA of FedAvg, RFA, MKrum, SignGuard, and our method AlignIns, in defending against ByzMean attack on CIFAR-10 dataset with attack ratios of $10\%$ and $20\%$ under different data settings. 
The results indicate that non-robust baseline FedAvg collapsed when facing to ByzMean attack in all cases, yielding an accuracy below $20\%$. RFA and MKrum provide a certain but limited Byzantine-robustness. In contrast, AlignIns consistently achieves comparable accuracy with SOTA SignGuard across all scenarios. These results demonstrate AlignIns' generalization ability for both backdoor and Byzantine attacks, making it a potential and potent method for practical application in real-world scenarios where there is no prior knowledge about the attack type.

\subsection{Effectiveness on More Datasets}
\label{apdx: more_dataset}
To validate that the achieved robustness by AlignIns can be generalized to other datasets, we show our evaluation results on MNIST, FMNIST, and Sentiment140 under Badnet attack in \autoref{tab:vary_dataset}. We also involve the perfectly robust FedAvg* for comparison. Notably, AlignIns consistently aligns with FedAvg* in MA, BA, and RA, indicating AlignIns can accurately identify malicious model updates and preserve benign model updates at the same time to attain such a high robustness and model performance. Additionally, AlignIns shows SOTA defense efficacy compared to other counterparts. For example, AlignIns maintains the highest BA at $0.36\%$, $0.01\%$, and $41.43\%$, with an improvement of $+21.42\%$, $+0.03\%$, and $+57.62\%$ over RLR on the respective three datasets. Besides, AlignIns also achieves the highest RA across all datasets, averaging a $+22.21\%$ increase compared to RFA. These findings verify the robustness and stability of AlignIns across various datasets. 
\begin{table}[ht]
  \centering
  \caption{Performance of AlignIns on More Datasets.}
  \scalebox{0.65}{
    \begin{tabular}{c|ccc|ccc|ccc}
    \toprule
    \multirow{2}[4]{*}{\textbf{Method}} & \multicolumn{3}{c}{\textbf{MNIST}} & \multicolumn{3}{c}{\textbf{FMNIST}} & \multicolumn{3}{c}{\textbf{Sentiment140}} \\
\cmidrule(r){2-4} \cmidrule(r){5-7} \cmidrule(r){8-10}          & MA$\uparrow$    & BA$\downarrow$    & RA$\uparrow$    & MA$\uparrow$    & BA$\downarrow$    & RA$\uparrow$    & MA$\uparrow$    & BA$\downarrow$    & RA$\uparrow$ \\
    \midrule
    FedAvg & 97.66 & 99.87 & 0.13  & 88.34 & 98.40 & 1.46  & 66.16 & 85.55 & 14.45 \\
    FedAvg* &   97.63    &   0.37    &  97.60     &   88.44    &   0.60    &   76.72    & 67.31 & 41.57 & 58.43 \\
    \midrule
    RLR   & 96.48 & 21.78 & 75.39 & 86.51 & \underline{0.04}  & \underline{75.44} & 51.23 & \underline{99.05} & \underline{0.95} \\
    RFA   & \underline{97.72} & \underline{0.61}  & \underline{97.53} & \textbf{88.53} & 13.08 & 69.09 & \underline{60.71} & 99.90 & 0.10 \\
    \rowcolor[rgb]{ .816,  .808,  .808} AlignIns & \textbf{97.76} & \textbf{0.36}  & \textbf{97.73} & \underline{88.50}  & \textbf{0.01}  & \textbf{77.04} & \textbf{69.26} & \textbf{41.43} & \textbf{58.57} \\
    \bottomrule
    \end{tabular}%
    }
  \label{tab:vary_dataset}%
\end{table}%

\subsection{Effectiveness under Various Attack Ratios.}
\label{apdx: attack_ratio}
We further evaluate the performance of AlignIns under various attack ratios in non-IID settings. We conduct the experiments under PGD and Scaling attacks with the attack ratio varying from $5\%$ to $30\%$ on non-IID CIFAR-10 and CIFAR-100 datasets. 
As shown in \autoref{fig:attackraio}, the RA of RLR and MKrum generally decreases as the attack ratio increases. For instance, when the attack ratio exceeds $20\%$, MKrum loses effectiveness, with RA dropping to as low as $0.02\%$. This decline is primarily due to the PGD attack, which projects malicious model updates within a sphere centered around the global model, limiting magnitude changes and evading detection by magnitude-based methods like MKrum. Lockdown achieves comparable robustness with AlignIns at low attack ratios on the CIFAR-10 dataset. Yet, it fails to effectively protect against both types of attacks when the attack ratios are high ($30\%$), resulting in considerable declines in robustness. Compared to its counterparts, AlignIns achieves a higher and more stable performance. As the attack ratio increases, AlignIns only has a minor decrease in RA.

\begin{figure}[ht]
    \centering
\includegraphics[width=0.90\linewidth]{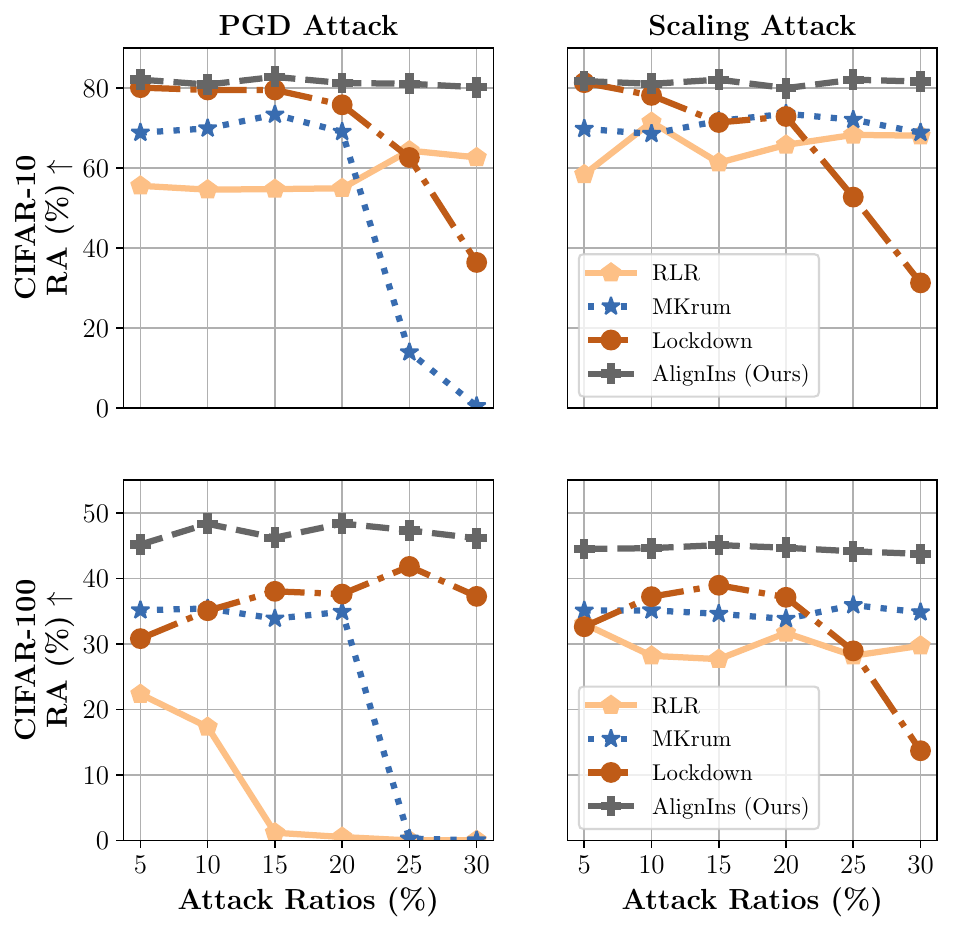}
    \vspace{-5pt}
    \caption{RA of AlignIns under various attack ratios on CIFAR-10 (upper row) and CIFAR-100 (lower row) datasets, compared with Lockdown, MKrum, and RLR.}
    \label{fig:attackraio}
    \vspace{-5pt}
\end{figure}

\section{Impact of Filtering Radii} \label{apdx:impact_of_filter_radii}

Here, we dive into the impact of different configurations of filtering radii, $\lambda_s$ and $\lambda_c$, on the efficacy of AlignIns. A smaller $\lambda_s$ or $\lambda_c$ indicates more stringent filtering and results in a smaller benign set for aggregation. We conduct the experiments on non-IID CIFAR-10 and CIFAR-100 datasets under Badnet and PGD attacks. The results, as detailed in \autoref{tab:filter}, show the ideal configurations of $\lambda_s$ and $\lambda_c$ that effectively balance the filtering intensity while maximizing the robustness of the model. Specifically, for CIFAR-10 dataset, the optimal RA is attained when $\lambda_s$ and $\lambda_c$ are both set to $1.0$ under both Badnet and PGD attacks, suggesting an ideal level of filtering intensity. A reduction in either $\lambda_s$ or $\lambda_c$ leads to a slight drop in RA, implying that some benign updates may be erroneously discarded due to an overly stringent filtering radius. In contrast, when $\lambda_s$ and $\lambda_c$ are increased to $2.0$, there's a significant decline in AlignIns' RA, due to the excessively permissive filtering threshold. As for CIFAR-100 dataset, AlignIns' performance remains stable against variations in both radii. Specifically, under the Badnet attack, AlignIns performs best when both radii are at $2.0$, while for the PGD attack, the radii at $1.0$ are most effective. This is mainly because PGD attack limits the large malicious model update changes, conducting a more stealthy attack than Badnet. By doing so, it makes the malicious model updates more similar to benign ones, leading to a smaller filter radius. 

\begin{table}[ht]
  \centering
  \caption{Performance of AlignIns with Different Filtering Radii.}
  \scalebox{0.7}{
    \begin{tabular}{cc|cc|cc|cc|cc}
    \toprule
    \multicolumn{2}{c}{\multirow{2}[4]{*}{\textbf{Config.}}} & \multicolumn{4}{c}{\textbf{CIFAR-10}}  & \multicolumn{4}{c}{\textbf{CIFAR-100}} \\
\cmidrule(r){3-6} \cmidrule(r){7-10}    \multicolumn{2}{c}{} & \multicolumn{2}{c}{Badnet} & \multicolumn{2}{c}{PGD} & \multicolumn{2}{c}{Badnet} & \multicolumn{2}{c}{PGD} \\
    \cmidrule(r){1-2} \cmidrule(r){3-4} \cmidrule(r){5-6} \cmidrule(r){7-8} \cmidrule(r){9-10}
    $\lambda_s$ & $\lambda_c$ & BA$\downarrow$    & RA$\uparrow$    & BA$\downarrow$    & RA$\uparrow$    & BA$\downarrow$    & RA$\uparrow$    & BA$\downarrow$    & RA$\uparrow$ \\
    \midrule
    0.5   & 0.5   & \textbf{0.58} & 76.37 & 3.29  & \underline{79.39} & 0.59  & 43.22 & 0.59  & 46.17 \\
    1.0   & 0.5   & 4.71  & 78.27 & 63.60 & 32.27 & \textbf{0.49} & 44.41 & 0.62  & 46.83 \\
    0.5   & 1.0   & 3.11  & \underline{78.99} & \textbf{1.73} & 79.37 & 0.58  & 43.18 & \textbf{0.19} & 44.67 \\
    1.0   & 1.0   & \underline{1.70}  & \textbf{81.32} & \underline{2.31}  & \textbf{81.18} & \underline{0.54}  & \underline{44.67} & \underline{0.52}  & \textbf{48.37} \\
    2.0   & 2.0   & 57.47 & 37.53 & 81.33 & 17.69 & 0.76  & \textbf{47.07} & 0.68  & \underline{46.99} \\
    \bottomrule
    \end{tabular}%
    }
  \label{tab:filter}%
\end{table}%

\section{Computational Cost of AlignIns}
\label{apdx: computational_cost}
We compare the computational cost of AlignIns with other counterparts. AlignIns calculates the $\mathrm{MPSA}$ metric using the Top-$k$ indicator, incurring a complexity of $O(d\log d)$ due to the use of sorting algorithms like \textit{merge sort} in the parameter space of the local update. 
As a result, the total computational expense of AlignIns in the worst-case scenario is $O(nd\log d)$. Nonetheless, we argue that the computational burden of AlignIns is comparable with several robust aggregation methods such as Krum and MKrum, both of which have a complexity of $O(dn^2)$, the Coordinate-wise median with $O(dn)$, and Trmean at $O(dn\log n)$. Each method shows a linear dependency on $d$, which can be considerably large in modern deep neural networks (i.e., $d \gg n$), and thus is the predominant factor in computational complexity. Empirically, AlignIns imposes minimal computational overhead on the server side ($0.13$ seconds per round), compared to $4.02$ seconds for another filtering-based method MM. Other methods like Lockdown introduce additional computational overhead on local clients, which is undesirable in many scenarios.

\clearpage
\onecolumn

\section{Proof preliminaries}
\subsection{Useful Inequalities}

\begin{lemma}Given any two vectors $a, b \in \mathbb{R}^d$,
\label{Lemma_for_product}
\begin{align*}
    2\left < a ,b \right > \leq \alpha \left \| a \right \|^2 + \frac{1}{\alpha} \left \| b \right \|^2, \forall \; \alpha > 0.
\end{align*} 
\end{lemma}

\begin{lemma}Given any two vectors $a, b \in \mathbb{R}^d$,
\label{Lemma_for_break_norm}
\begin{align*}
    \left\| a + b \right \|^2 \leq (1+\delta)\left\| a \right \|^2 + (1 + \delta ^{-1})\left\| b \right \|^2, \forall \; \delta > 0.
\end{align*} 
\end{lemma}
\begin{lemma} \label{Lemma_for_extract_sum}Given arbitrary set of n vectors $\{a_i\}^n_{i=1}$, $a_i \in \mathbb{R}^d$,
\begin{align*}
    \left\|  \sum^n_{i=1} a_i \right \|^2 \leq n\sum^n_{i=1}\left \| a_i\right \|^2.
\end{align*}
\end{lemma}
\begin{lemma}\label{local_divergence}
    If the learning rate $\eta\leq 1/2\tau$, under \autoref{ass2} and \autoref{ass3}, the local divergence of benign model updates are bounded as follows:
    \begin{equation}
        \frac{1}{|\mathcal{B}|} \sum_{i \in \mathcal{B}} \mathbb{E}\left\|  {\Delta}_i - \bar{\Delta}_{\mathcal{B}} \right\|^2 \leq 2\bar{\nu} + \bar{\zeta} \nonumber
    \end{equation}
\end{lemma}
\begin{proof}
Given that $\Delta_i = \eta \sum^{\tau-1}_{s=0} g^{s}_i$ where $\eta$ is the learning rate and $g^{s}_i$ is the local stochastic gradient over the mini-batch $s$. We have
\begin{align}
    \nonumber \frac{1}{|\mathcal{B}|}\sum_{i \in \mathcal{B}} \mathbb{E}\left \|  \Delta_i - \bar{\Delta}_\mathcal{B} \right \|^2 
    &= \frac{1}{|\mathcal{B}|}\sum_{i \in \mathcal{B}} \mathbb{E}\left \|  \eta \sum^{\tau - 1}_{s=0} g^{s}_i - \frac{1}{|\mathcal{B}|}\sum_{i \in \mathcal{B}}\eta\sum^{\tau - 1}_{s=0} g^{s}_i \right \|^2 \\
    & =\frac{\eta^2}{|\mathcal{B}|}\sum_{i \in \mathcal{B}} \mathbb{E} \left \|  \sum^{\tau - 1}_{s=0} g^{s}_i - \frac{1}{|\mathcal{B}|}\sum_{i \in \mathcal{B}} \sum^{\tau - 1}_{s=0} g^{s}_i \right \|^2 \nonumber \\
    & \leq \frac{\tau \eta^2}{|\mathcal{B}|}\sum_{i \in \mathcal{B}}\sum^{\tau - 1}_{s=0} \mathbb{E}\left \|  g^{s}_i - \frac{1}{|\mathcal{B}|}\sum_{i \in \mathcal{B}}  g^{s}_i \right \|^2 \nonumber \\
    & = \frac{\tau \eta^2}{|\mathcal{B}|}\sum_{i \in \mathcal{B}}\sum^{\tau - 1}_{s=0} \mathbb{E}\left \|  \left( g^{s}_i - \nabla \mathcal{L}_i(\theta^{s}_i) \right)  + \left( \nabla \mathcal{L}_\mathcal{B}(\theta^{s}_i) - \frac{1}{|\mathcal{B}|}\sum_{i \in \mathcal{B}}  g^{s}_i \right) + \left( \nabla \mathcal{L}_i(\theta^{s}_i) -\nabla \mathcal{L}_\mathcal{B}(\theta^{s}_i) \right )  \right \|^2 \nonumber \\
    & \leq \frac{3 \tau \eta^2}{|\mathcal{B}|}\sum_{i \in \mathcal{B}}\sum^{\tau - 1}_{s=0} \underbrace{\mathbb{E} \left \|  g^{s}_i - \nabla \mathcal{L}_i(\theta^{s}_i) )\right \|^2}_{T_1}  + \frac{3 \tau \eta^2}{|\mathcal{B}|}\sum_{i \in \mathcal{B}}\sum^{\tau - 1}_{s=0}\underbrace{\mathbb{E} \left \| \nabla \mathcal{L}_\mathcal{B}(\theta^{s}_i) - \frac{1}{|\mathcal{B}|}\sum_{i \in \mathcal{B}}  g^{s}_i \right \|^2}_{T_2} \nonumber \\
    & \quad + \underbrace{\frac{3 \tau \eta^2}{|\mathcal{B}|}\sum_{i \in \mathcal{B}}\sum^{\tau - 1}_{s=0}\mathbb{E}\left \|  \nabla \mathcal{L}_i(\theta^{s}_i) -\nabla \mathcal{L}_\mathcal{B}(\theta^{s}_i)   \right \|^2}_{T_3}, \label{proof_our_kappa_t1}
\end{align}
where the first inequality follows \autoref{Lemma_for_extract_sum}, and the last second follows \autoref{Lemma_for_break_norm}. For $T_1$, with \autoref{ass2}, we have
\begin{align}
    T_1 \leq \bar{\nu}. \label{proof_our_kappa_t2}
\end{align}
For $T_2$, we have
\begin{align}
    T_2 = \mathbb{E}\left \| \nabla \mathcal{L}_\mathcal{B}(\theta_i^{ s}) - \frac{1}{|\mathcal{B}|}\sum_{i \in \mathcal{B}}  g^{ s}_i \right \|^2 = \mathbb{E}\left \| \frac{1}{|\mathcal{B}|}\sum_{i \in \mathcal{B}} \left(\nabla \mathcal{L}_i(\theta_i^{ s}) -  g^{ s}_i\right) \right \|^2  
    \leq \frac{1}{|\mathcal{B}|}\sum_{i \in \mathcal{B}} \mathbb{E}\left \| \nabla \mathcal{L}_i(\theta_i^{ s}) -  g^{s}_i \right \|^2 
    \leq  \bar{\nu},\label{proof_our_kappa_t3}
\end{align}
where the first inequality follows \autoref{Lemma_for_extract_sum}, and the last inequality follow \autoref{ass2}. For $T_3$, by \autoref{ass3}, we have
\begin{align}
    T_3 = \frac{3 \tau \eta^2}{|\mathcal{B}|}\sum_{i \in \mathcal{B}}\sum^{\tau - 1}_{s=0}\mathbb{E}\left \|  \nabla \mathcal{L}_i(\theta^{ s}_i) -\nabla \mathcal{L}_\mathcal{B}(\theta_i^{ s})   \right \|^2
    \leq 3 \tau \eta^2\sum^{\tau - 1}_{s=0} \bar{\zeta} = 3 \tau^2 \eta^2 \bar{\zeta}.\label{proof_our_kappa_t4}
\end{align}
\noindent Plugging \ineqautoref{proof_our_kappa_t2}, \ineqautoref{proof_our_kappa_t3}, and \ineqautoref{proof_our_kappa_t4} back to \ineqautoref{proof_our_kappa_t1}, with $\eta\leq 1/2\tau$, we have
\begin{align}
    \frac{1}{|\mathcal{B}|}\sum_{i \in \mathcal{B}} \mathbb{E}\left \|  \Delta_i - \bar{\Delta}_\mathcal{B} \right \|^2 \leq 3 \tau^2 \eta^2 (2\bar{\nu} + \bar{\zeta})\leq 2\bar{\nu} + \bar{\zeta}.
\end{align}
This concludes the proof.
\end{proof}

\subsection{Proof of \autoref{lemma:AlignInskappa}}
\label{apdx: proof_lemma_1}
\begin{proof}
Recall that our method is denoted by $F\colon \mathbb{R}^{d \times n} \rightarrow \mathbb{R}^d$. Given that $\Delta^t = F(\Delta^t_1, \Delta^t_2, \dots, \Delta^t_n) = 1/|\mathcal{S}^t|\sum_{i \in \mathcal{S}^t}\Delta^t_i$ where $\mathcal{S}^t$ is the selected set by $F$ in round $t$ and $m < n / 2$. Let $\Delta^t_\mathcal{B} = 1/|\mathcal{B}|\sum_{i \in \mathcal{B}}\Delta^t_i$ be the average of benign updates in round $t$, where $|\mathcal{B}| = n-m$. We have
\begin{align}
        \mathbb{E} \left  \| \Delta^t - \Delta^t_\mathcal{B} \right \|^2 &= \mathbb{E}\left \| \frac{1}{|\mathcal{S}^t|} \sum_{i \in \mathcal{S}^t} ( \Delta^t_i - \Delta^t_\mathcal{B} ) \right \|^2 \leq \mathbb{E}\frac{1}{|\mathcal{S}^t|} \sum_{i \in \mathcal{S}^t} \left \|  \Delta^t_i - \Delta^t_\mathcal{B} \right \|^2,
\end{align}
where the first inequality follows \autoref{Lemma_for_extract_sum}.

\noindent If $\mathcal{S}^t \subseteq \mathcal{B}$, thus $\mathcal{S}^t \backslash \mathcal{B} = \emptyset$ and $\mathcal{B} \backslash \mathcal{S}^t \subseteq \mathcal{B}$ we have
\begin{align}
    \mathbb{E} \left  \| \Delta^t - \Delta^t_\mathcal{B} \right \|^2 &\leq \mathbb{E}\frac{1}{|\mathcal{S}^t|} \sum_{i \in \mathcal{S}^t} \left \|  \Delta^t_i - \Delta^t_\mathcal{B} \right \|^2 \leq \mathbb{E}\frac{1}{|\mathcal{S}^t|} \sum_{i \in \mathcal{B}} \left \|  \Delta^t_i - \Delta^t_\mathcal{B} \right \|^2 \nonumber \\
    &\leq \frac{|\mathcal{B}|}{|\mathcal{S}^t|} \left (2\bar{\nu} + \bar{\zeta} \right) \nonumber \\
    &= \frac{n-m}{|\mathcal{S}^t|} \left (2\bar{\nu} + \bar{\zeta} \right), \label{kappa_in_1}
\end{align}
where the last inequality follows \autoref{local_divergence}.

\noindent If $\mathcal{S} \nsubseteq \mathcal{B}$, we let $\mathcal{S} \backslash \mathcal{B} = \mathcal{R}$, where $|\mathcal{R}| \leq m$, and $\mathcal{S} \cap \mathcal{B} = \mathcal{P}$, one yields 
\begin{align}
     \mathbb{E} \left  \| \Delta^t - \Delta^t_\mathcal{B} \right \|^2 &\leq \mathbb{E}\frac{1}{|\mathcal{S}^t|} \sum_{i \in \mathcal{S}^t} \left \|  \Delta^t_i - \Delta^t_\mathcal{B} \right \|^2 = \mathbb{E}\frac{1}{|\mathcal{S}^t|} \left[ \sum_{i \in \mathcal{P}} \left \|  \Delta^t_i - \Delta^t_\mathcal{B} \right \|^2 + \sum_{i \in \mathcal{R}} \left \|  \Delta^t_i - \Delta^t_\mathcal{B} \right \|^2 \right] \nonumber \\
        &= \mathbb{E}\frac{1}{|\mathcal{S}^t|} \left[ \sum_{i \in \mathcal{R}} \left \|  \Delta^t_i - \Delta^t_\mathcal{P} + \Delta^t_\mathcal{P} - \Delta^t_\mathcal{B} \right \|^2 + \sum_{i \in \mathcal{P}} \left \|  \Delta^t_i - \Delta^t_\mathcal{B} \right \|^2 \right]\nonumber \\
        &\leq \mathbb{E}\frac{1}{|\mathcal{S}^t|} \left[ 2\sum_{i \in \mathcal{R}} \left \|  \Delta^t_i - \Delta^t_\mathcal{P} \right \|^2 + 2\sum_{i \in \mathcal{R}}\left \| \Delta^t_\mathcal{P} - \Delta^t_\mathcal{B} \right \|^2 + \sum_{i \in \mathcal{P}} \left \|  \Delta^t_i - \Delta^t_\mathcal{B} \right \|^2 \right],
    \end{align}
    where the first inequality follows \autoref{Lemma_for_break_norm}. 
    
    \noindent Due to the use of MZ-score, models in $\mathcal{S}^t$ are centered around the median within a $\lambda_c$ (and $\lambda_s$) radius. If the radius parameter $\lambda_c$ or $\lambda_s$ equals zero, only the median model (based on Cosine similarity or masked principal sign alignment ratio) will be selected for averaging. To maximize benign model inclusion in averaging, we assume the radius parameters $\lambda_c$ and $\lambda_s$ are set sufficiently large to ensure $|\mathcal{S}^t|\geq n -2m$. More precisely, assume there exist two positive constants $\lambda^{+}_c$ and $\lambda^{+}_s$, and if the radius parameters $\lambda_c$ and $\lambda_s$ in \algautoref{alg:main} satisfy $\lambda_c \geq \lambda^{+}_c, \lambda_s \geq \lambda^{+}_s$, we have $|\mathcal{S}^t|\geq n -2m$. Additionally, if $m < n/(3 + \epsilon)$, we can have at least one benign clients in $\mathcal{S}^t$ and the ratio of $|\mathcal{R}| / |\mathcal{P}|$ is bounded by $1/\epsilon$. 
    Consequently, we have
    \begin{align}
        \mathbb{E} \left  \| \Delta^t - \Delta^t_\mathcal{B} \right \|^2&\leq \mathbb{E}\frac{1}{|\mathcal{S}^t|} \left[ 2\sum_{i \in \mathcal{R}} \left [ \frac{1}{|\mathcal{P}|} \sum_{j \in \mathcal{P}} \left \|  \Delta^t_i - \Delta^t_j \right \|^2 \right] + \frac{2 |\mathcal{R} |}{|\mathcal{P}|} \sum_{i \in \mathcal{P}} \left \| \Delta^t_i - \Delta^t_\mathcal{B} \right \|^2 + \sum_{i \in \mathcal{P}} \left \|  \Delta^t_i - \Delta^t_\mathcal{B} \right \|^2 \right]\nonumber \\
        &\leq \mathbb{E}\frac{1}{|\mathcal{S}^t|} \left[ 8 |\mathcal{R}| c^2 + \left ( \frac{2 |\mathcal{R} |}{|\mathcal{P}|} + 1 \right) \sum_{i \in \mathcal{P}}\left \| \Delta^t_i- \Delta^t_\mathcal{B} \right \|^2 \right]\nonumber \\
        &\leq \mathbb{E}\frac{1}{|\mathcal{S}^t|} \left[ 8 |\mathcal{R}| c^2 + \left ( \frac{2 |\mathcal{R} |}{|\mathcal{P}|} + 1 \right) |\mathcal{B}|  (2 \bar{\nu} + \bar{\zeta})\right]\nonumber \\
        &= \frac{|\mathcal{B}|}{|\mathcal{S}^t|}\left ( \frac{2 |\mathcal{R} |}{|\mathcal{P}|} + 1 \right)   (2 \bar{\nu} + \bar{\zeta})   + \frac{8 |\mathcal{R}| c^2}{|\mathcal{S}^t|} \nonumber \\
        &\leq \frac{|\mathcal{B}|}{|\mathcal{S}^t|}\left ( \frac{2}{\epsilon} + 1 \right)   (2 \bar{\nu} + \bar{\zeta})   + \frac{8 |\mathcal{R}| c^2}{|\mathcal{S}^t|}, \label{kappa_in_2}
\end{align}
where the first inequality follows \autoref{Lemma_for_extract_sum}, the second inequality holds as the model updates in $\mathcal{S}^t$ is bounded by $c$, the third inequality follows \autoref{local_divergence}. 

\noindent Summarizing \ineqautoref{kappa_in_1} and \ineqautoref{kappa_in_2}, we have
\begin{align}
    \mathbb{E} \left  \| \Delta^t - \Delta^t_\mathcal{B} \right \|^2 & \leq \nonumber \begin{cases}
            \frac{ n-m}{n-2m} \left (2\bar{\nu} + \bar{\zeta} \right) , & \text{if} \ \mathcal{S}^t \subseteq \mathcal{B}\\ \nonumber \\
            \frac{ n-m}{n-2m} \left (\frac{2}{\epsilon} + 1\right)(2\bar{\nu} + \bar{\zeta})  + \frac{8mc^2}{n-2m}, & \text{if} \ \mathcal{S}^t \nsubseteq \mathcal{B} 
        \end{cases}\nonumber\\ \nonumber \\
        &\leq \frac{ n-m}{n-2m} \left (\frac{2}{\epsilon} + 1\right)(2\bar{\nu} + \bar{\zeta})  + \frac{8mc^2}{n-2m} \nonumber \\
        & \leq \left( 1 + \frac{m}{n-2m} \right) \left( \left (\frac{2}{\epsilon} + 1 \right) (2\bar{\nu} + \bar{\zeta}) + 8 c^2 \right),
\end{align}
which concludes the proof.
\end{proof}

\subsection{Proof of \autoref{lemma:AlignInscertified_radius}}
\label{apdx: proof_of_pro_error}
\begin{proof}
    We use $\theta$ to denote the model trained over $[n]$ which contains $\mathcal{B} \in [n], \mathcal{M} \in [n]$ where $\mathcal{B}$ is the set of benign clients and $\mathcal{M}$ is the set of malicious clients. Obviously, $\mathcal{B} \cup \mathcal{M} = [n]$ and $\mathcal{B} \cap \mathcal{M} = \emptyset $. We use $\theta^*$ to denote the clean model which is trained over $\mathcal{B}$. The update rules for $\theta$ and $\theta^*$ are as follows.
\begin{align}
    \theta^{t+1} &= \theta^t - \alpha\Delta^t \label{norupdaterule} \\
 \theta^{t+1, *} &= \theta^{t, *} - \alpha\Delta^{t, *} \label{opupdatrule}.
\end{align} 
With \eqautoref{norupdaterule} and \eqautoref{opupdatrule}, we have
    \begin{align}
        \left \| \theta^{t+1} - \theta^{t+1, *} \right\|^2 &= \left \| \theta^t - \alpha \Delta^t - (\theta^{t, *} - \alpha\Delta^{t, *}) \right \|^2 \nonumber \\
        &= \left \| \theta^t - \theta^{t, *} + \alpha \Delta^t - \alpha\Delta^{t, *} \right \|^2 \nonumber \\
        &\leq 2\left \| \theta^t - \theta^{t, *}  \right \|^2 +\underbrace{2 \alpha^2\left \| \Delta^t - \Delta^{t, *}  \right \|^2}_{T_1},
    \end{align}
    where the first inequality follows \autoref{Lemma_for_break_norm}.
    
    \noindent Now, we treat $T_1$. As $\Delta^{t, *} = 1/|\mathcal{B}| \sum_{i \in \mathcal{B} }\Delta^{t, *}_i$, let $\Delta^t_\mathcal{B} = 1/|\mathcal{B}| \sum_{i \in \mathcal{B} }\Delta^{t}_i$, we have
    \begin{align}
        T_1 &= 2\alpha^2\left \| \Delta^t - \Delta^t_\mathcal{B} + \Delta^t_\mathcal{B} - \Delta^{t, *} \right \|^2 \nonumber \\
        &\leq \underbrace{4 \alpha^2 \left \|  \Delta^t - \Delta^t_\mathcal{B} \right \|^2}_{T_2} + \underbrace{4 \alpha^2 \left \|  \Delta^t_\mathcal{B} - \Delta^{t, *} \right \|^2}_{T_3}, \label{inter_5}
    \end{align}
    where the first inequality follows \autoref{Lemma_for_break_norm}.
    
    \noindent We now treat $T_2$, $T_3$, respectively. For $T_2$, given that $\Delta^t = F(\Delta^t_1, \Delta^t_2, \dots, \Delta^t_n)$, we have
    \begin{align}
        T_2 = 4 \alpha^2 \left \|  \Delta^t - \Delta^t_\mathcal{B} \right \|^2 \leq 4\alpha^2 \kappa, \label{inter_4}
    \end{align}
    where the first inequality follows \autoref{lemma:AlignInskappa} in the paper. Define $ \Delta_{\mathcal{B}}:= \frac{1}{|\mathcal{B}|}\sum_{i\in\mathcal{B}} \Delta^t_i = \frac{1}{|\mathcal{B}|}\sum_{i\in\mathcal{B}} \eta g^t_i$ . For $T_3$, we have
    \begin{align}
        T_3 &=4 \alpha^2\left \|  \Delta^t_\mathcal{B} - \Delta^{t, *} \right \|^2 = 4 \alpha^2\eta^2 \left \| g^t_\mathcal{B} - g^{t, *} \right \|^2 \nonumber \\
        &=4 \alpha^2\eta^2 \left \| g^t_\mathcal{B} - \nabla \mathcal{L}_{\mathcal{B}}(\theta^t) + \nabla \mathcal{L}_{\mathcal{B}}(\theta^t) - g^{t, *} - \nabla \mathcal{L}_{\mathcal{B}}(\theta^{t, *}) + \nabla \mathcal{L}_{\mathcal{B}}(\theta^{t,*})\right \|^2 \nonumber \\
        &=4 \alpha^2\eta^2 \left \| g^t_\mathcal{B} - \nabla \mathcal{L}_{\mathcal{B}}(\theta^t) - \left(g^{t, *}- \nabla \mathcal{L}_{\mathcal{B}}(\theta^{t,*})\right) + \nabla \mathcal{L}_{\mathcal{B}}(\theta^t) - \nabla \mathcal{L}_{\mathcal{B}}(\theta^{t, *}) \right \|^2 \nonumber \\
        &\leq \underbrace{12 \alpha^2\eta^2\left \| g^t_\mathcal{B} - \nabla \mathcal{L}_{\mathcal{B}}(\theta^t)\right \|^2}_{T_4} + \underbrace{12 \alpha^2\eta^2 \left \| \left(g^{t, *}- \nabla \mathcal{L}_{\mathcal{B}}(\theta^{t,*})\right)\right \|^2}_{T_5} + \underbrace{12 \alpha^2\eta^2\left \| \nabla \mathcal{L}_{\mathcal{B}}(\theta^t) - \nabla \mathcal{L}_{\mathcal{B}}(\theta^{t, *}) \right \|^2}_{T_6}, \label{inter_3}
    \end{align}
    where the first inequality follows \autoref{Lemma_for_break_norm}.
    For $T_4$, we have
    \begin{align}
        T_4 &= 12 \alpha^2\eta^2\left \| g^t_\mathcal{B} - \nabla \mathcal{L}_{\mathcal{B}}(\theta^t)\right \|^2 \nonumber 
        = 12 \alpha^2\eta^2 \left \| \frac{1}{|\mathcal{B}|} \sum_{i \in \mathcal{B}} g^t_i - \frac{1}{|\mathcal{B}|} \sum_{i \in \mathcal{B}} \nabla \mathcal{L}_i(\theta^t_i) \right \|^2 \nonumber = 12 \alpha^2\eta^2 \left \| \frac{1}{|\mathcal{B}|} \sum_{i \in \mathcal{B}} \left ( g^{t}_i -  \nabla \mathcal{L}_i(\theta^t_i) \right ) \right \|^2  \nonumber \\
        &\leq \frac{12 \alpha^2\eta^2}{|\mathcal{B}|} \sum_{i \in \mathcal{B}} \left \| g^t_i - \nabla \mathcal{L}_i (\theta^t_i) \right \|^2 
        = \frac{12 \alpha^2\eta^2}{|\mathcal{B}|} \sum_{i \in \mathcal{B}} \left \| \sum^{\tau - 1}_{s=0}g^{t, s}_i - \sum^{\tau - 1}_{s=0}\nabla \mathcal{L}_i (\theta^{t, s}_i) \right \|^2 \nonumber \\
        &\leq \frac{12 \alpha^2\tau \eta^2}{|\mathcal{B}|} \sum_{i \in \mathcal{B}}\sum^{\tau - 1}_{s=0} \left \| g^{t, s}_i - \nabla \mathcal{L}_i (\theta^{t, s}_i) \right \|^2 
        \leq 12 \alpha^2\tau \eta^2 \sum^{\tau - 1}_{s=0} \bar{\nu} \nonumber \\
        &= 12 \alpha^2\tau^2 \eta^2 \bar{\nu},\label{proof_lemma_cr_t4}
    \end{align}
    where the both first and second inequality follow \autoref{Lemma_for_extract_sum}, the third inequality follows \autoref{ass2}. 
    
    \noindent Similarly, we have 
    \begin{align}
        T_5 \leq 12 \alpha^2\tau^2 \eta^2 \bar{\nu}. \label{proof_lemma_cr_t5}
    \end{align}
    For $T_6$, we have
    \begin{align}
        T_6 = 12 \alpha^2\eta^2 \left \| \nabla \mathcal{L}_{\mathcal{B}}(\theta^t) - \nabla \mathcal{L}_{\mathcal{B}}(\theta^{t, *}) \right \|^2 
        \leq 12 \alpha^2\eta^2 \mu^2 \left \| \theta^t - \theta^{t, *} \right \|^2, \label{proof_lemma_cr_t6}
    \end{align}
    where the first inequality follows \autoref{ass1}. 
    
    \noindent Plugging \ineqautoref{proof_lemma_cr_t6}, \ineqautoref{proof_lemma_cr_t5}, and \ineqautoref{proof_lemma_cr_t4} back to \ineqautoref{inter_3}, we have:
    \begin{align}
        T_3 \leq 24 \alpha^2\tau^2 \eta^2 \bar{\nu} + 12 \alpha^2\eta^2 \mu^2 \left \| \theta^t - \theta^{t, *} \right \|^2. \label{inter_6}
    \end{align}
    Plugging \ineqautoref{inter_6}, \ineqautoref{inter_4} back to \ineqautoref{inter_5}, we have
    \begin{align}
        T_1 \leq 4 \alpha^2 \kappa + 24 \alpha^2 \tau^2 \eta^2 \bar{\nu} + 12\alpha^2 \eta^2 \mu^2 \left \| \theta^t - \theta^{t, *} \right \|^2.
    \end{align}
    Therefore, we have
    \begin{align}
        \left \| \theta^{t+1} - \theta^{t+1, *} \right\|^2 
        &\leq 2\left \| \theta^t - \theta^{t, *}  \right \|^2 +4 \alpha^2 \kappa + 24 \alpha^2 \tau^2 \eta^2 \bar{\nu} + 12\alpha^2 \eta^2 \mu^2 \left \| \theta^t - \theta^{t, *} \right \|^2 \nonumber \\
        & = (2 + 12 \alpha^2 \eta^2 \mu^2) \left \| \theta^t - \theta^{t, *} \right \|^2 + 4 \alpha^2 (\kappa + 6 \tau^2 \eta^2 \bar{\nu}) \nonumber \\
        &\leq (2 + 3 \alpha^2 \tau^{-2} \mu^2) \left \| \theta^t - \theta^{t, *} \right \|^2 + 4 \alpha^2 (\kappa +2 \bar{\nu}) \nonumber \\
        &\leq (2 + 3 \alpha^2 \mu^2) \left \| \theta^t - \theta^{t, *} \right \|^2 + 4 \alpha^2 (\kappa +2 \bar{\nu}), \label{convergence_final}
    \end{align}
    where the second inequality follows $\eta \leq 1/2\tau$, and the last inequality holds as $\tau^{-2} \leq 1$.
    
    \noindent We inductively prove the \autoref{lemma:AlignInscertified_radius}, assume for $T-1$ the statement of Lemma holds. Let $\phi(T) = \sum_{i=1}^ T (\alpha^i)^2$, by \ineqautoref{convergence_final}, we have
    \begin{align}
         \left \| \theta^{T} - \theta^{T, *} \right\|^2 \leq (2 + 3 \mu^2 (\alpha^T)^2) \phi(T-1)(2 + 3 \mu^2 )^{\phi(T-1)}(\kappa + 2 \bar{\nu}) + (\kappa + 2 \bar{\nu})(\alpha^T)^2.
    \end{align}
    \noindent By Bernoulli's inequality we have
    \begin{align}
        \left \| \theta^{T} - \theta^{T, *} \right\|^2 &\leq \phi(T-1)(2 + 3 \mu^2 )^{\phi(T-1) + (\alpha^T)^2}(\kappa + 2 \bar{\nu}) + (\kappa + 2 \bar{\nu})(\alpha^T)^2 \nonumber \\
        &= \phi(T-1)(2 + 3 \mu^2 )^{\phi(T) }(\kappa + 2 \bar{\nu}) + (\kappa + 2 \bar{\nu})(\alpha^T)^2 \nonumber \\
        &\leq (\phi(T-1) + (\alpha^T)^2)(2 + 3 \mu^2 )^{\phi(T) }(\kappa + 2 \bar{\nu}) \nonumber \\
        &\leq \phi(T)(2 + 3\mu^2 )^{\phi(T) }(\kappa + 2 \bar{\nu}),
    \end{align}
    which concludes the proof.
    \end{proof}

\end{document}